\documentclass{article} 
\usepackage{iclr2027_conference,times}

\usepackage{amsmath,amsfonts,bm}

\def\eqref#1{equation~\ref{#1}}

\def\1{\bm{1}}

\DeclareMathAlphabet{\mathsfit}{\encodingdefault}{\sfdefault}{m}{sl}
\SetMathAlphabet{\mathsfit}{bold}{\encodingdefault}{\sfdefault}{bx}{n}

\newcommand{\R}{\mathbb{R}}

\usepackage{microtype}
\usepackage{hyperref}
\usepackage{url}
\usepackage{multirow} 
\usepackage{booktabs}
\usepackage{caption}
\usepackage{subcaption}
\usepackage{longtable}
\usepackage{amsmath, amssymb, mathtools, amsthm, thm-restate}
\usepackage{algorithm}
\usepackage[noend]{algpseudocode}
\usepackage[dvipsnames]{xcolor}
\usepackage{tabularx}
\theoremstyle{plain}
\newtheorem{theorem}{Theorem}[section]

\newtheorem{lemma}[theorem]{Lemma}

\theoremstyle{definition}
\newtheorem{definition}[theorem]{Definition}

\theoremstyle{remark}
\newtheorem{remark}[theorem]{Remark}
\usepackage{enumitem}
\usepackage{makecell}

\newcommand{\simplex}{\mathbf{\Delta}}

\newcommand{\opt}{\mathbf{opt}}

\definecolor{DarkForestGreen}{cmyk}{0.91,0,0.88,0.25} 

\usepackage{lineno}

\definecolor{darkblue}{rgb}{0, 0, 0.5}
\hypersetup{colorlinks=true, citecolor=darkblue, linkcolor=darkblue, urlcolor=darkblue}

\title{Many Preferences, Few Policies: \\ Compact Portfolios for Multi-Objective \\ LLM Alignment}

\author{
Cheol Woo Kim\thanks{Equal contribution} \\
Harvard University \\
\And
Jai Moondra\footnotemark[\value{footnote}] \\
Carnegie Mellon University \\
\And
Roozbeh Nahavandi \\
The Ohio State University \\
\AND
Andrew Perrault \\
The Ohio State University \\
\And
Milind Tambe \\
Harvard University \\
\And
Swati Gupta \\
Massachusetts Institute of Technology
}

\iclrfinalcopy 
\begin{document}


{\centering 
    \maketitle
}
\lhead{Preprint.}

\begin{abstract}

Aligning large language models (LLMs) requires balancing competing objectives such as helpfulness, harmlessness, and conciseness. The appropriate balance varies across users and applications, yet training, evaluating, and deploying many policies across different reward weights is costly. We study how to identify a small portfolio of LLMs that preserves near-optimal performance across all reward weightings. We propose {\sc PALM} ({\bf P}ortfolio of {\bf A}ligned L{\bf LM}s), an algorithm that combines a structured grid of weight vectors,
a lazy search that optimizes policies only where needed, and pruning. Given target approximation tolerances,
{\sc PALM} returns a portfolio that provably contains a near-optimal policy for every
weight vector, with an explicit upper bound on portfolio size. Such portfolios can support scalable personalization, reward-weight exploration during model
development, and compact decoding-time  configurations. Experiments show that {\sc PALM} generally achieves smaller approximation gaps than
same-size portfolios built from uniformly spaced or randomly sampled weights.
We further demonstrate that {\sc PALM} scales effectively to higher-dimensional reward spaces through efficient search and sparse preference structure.

\end{abstract}

\section{Introduction}

Aligning large language models (LLMs) requires balancing competing objectives such as helpfulness, harmlessness, and conciseness, with the appropriate balance depending on the intended use \citep{wu2023morlhf}. Users may differ in their preferences \citep{sorensen24, li2025surveypersonalizationragagent}, while developers may seek a particular balance for a single model that will serve an entire application \citep{touvron2023llama2openfoundation,mu2024rule}. Understanding and selecting among these trade-offs is therefore a central challenge in multi-objective alignment.


We study this challenge within the widely used paradigm of multi-objective alignment with linear scalarization, where a weight vector represents a preference by specifying the relative importance of a fixed set of reward functions. Given a weight vector, an aligned policy can be obtained by optimizing the weighted combination of rewards, possibly with a regularization term. Existing approaches obtain such policies through fine-tuning \citep{wu2023morlhf, zhou2024modpo}, parameter interpolation \citep{rame2023rewarded}, or decoding-time alignment \citep{shi2024decoding, wang2024conditional}. Varying the weights induces a potentially infinite family of policies, each representing a different choice of trade-offs. 

However, practical constraints limit how many of these policies can be trained, evaluated, or deployed, so they must be chosen carefully. The effect of changing weights on LLM behavior is often hard to predict a priori, and ad hoc exploration may overlook valuable alternatives or produce redundant policies. These challenges motivate identifying a compact portfolio of policies that faithfully represents the available reward trade-offs and provides a manageable set of alternatives for comparison and selection. We therefore ask:

\vspace{-0.2cm}
\begin{center}
\emph{
How can we identify a small portfolio of LLMs that preserves near-optimal performance across all reward weightings, and how can the trade-off between portfolio size and approximation quality be explicitly controlled?
}
\end{center}
\vspace{-0.2cm}

Our key observation is that, although the space of weight vectors is continuous, a policy optimized for one scalarized objective can be provably near-optimal for many other weight vectors. Consequently, a finite set of well-chosen policies can approximate the entire family of optimal scalarized values. We propose PALM (Portfolio of Aligned LLMs), an algorithm that constructs such a portfolio: for every normalized weight vector, at least one policy in the portfolio is approximately optimal for the corresponding objective. The portfolio thus provides a compact representation of the alignment trade-offs measured by the specified reward functions.

PALM constructs a structured grid of weight vectors, selectively obtains policies through a lazy search procedure, and prunes redundant policies. PALM takes target approximation tolerances as input and returns a portfolio that meets them for every weight vector. We prove an upper bound on the size of this portfolio. The bound grows as the tolerances tighten, which makes explicit the trade-off between portfolio compactness and approximation quality. Our framework accommodates regularized and unregularized objectives, and we further extend it to sparse preferences, where most weight falls on a small subset of reward dimensions. 

Our portfolio framework provides three direct benefits. First, it enables deployment-efficient personalization across a user base with diverse preferences. Training and maintaining many separately fine-tuned policies for different weight vectors introduces substantial training, storage, serving, and operational overhead \citep{Bommasani2021FoundationModels, anthropic2026opus3, sun2024llumnix}. A provider can instead deploy a small, fixed portfolio of fine-tuned LLMs, with the assurance that every weight vector (and hence, every user) has a near-optimal option in the portfolio. 


Second, PALM can support reward-weight exploration during model development. A developer may seek a single model with a desired balance of qualities \citep{touvron2023llama2openfoundation,mu2024rule} but be unable to specify suitable reward weights before observing the resulting policies. In practice, this can require training and evaluating models under multiple reward combinations; prior work explicitly searches over reward weights or identifies their manual tuning as a challenge in multi-objective RLHF \citep{wu2023morlhf, xu2024perfectblendredefiningrlhf}. PALM provides a principled set of candidate policies that the developer can inspect and compare before selecting a final model. Evaluating each candidate, for example through human review or red-teaming, can itself be expensive \citep{perez-etal-2022-red, zheng2023judging},  so a small set limits inspection cost.  

Third, even when weights can be changed without additional training, as in decoding-time methods \citep{shi2024decoding}, a compact catalogue of decoding configurations (i.e., representative weight vectors) remains useful. A provider can evaluate this small set, through reward profiles and example responses, to see which reward trade-offs the method can achieve without sweeping the entire weight space. It can also serve as a menu of presets for users who cannot specify their weights directly.

We evaluate PALM with three methods for obtaining policies: full fine-tuning \citep{shao2024grpo}, parameter-efficient fine-tuning \citep{hu2022lora}, and decoding-time alignment \citep{shi2024decoding}. Across these settings, we compare portfolios at matched sizes using their approximation gaps over held-out reward weightings. PALM generally achieves better approximation quality than baselines built from uniformly spaced or randomly sampled weights. We further show that {\sc PALM} scales to higher-dimensional reward spaces. The lazy search optimizes far fewer policies than the grid contains, and a portfolio built on a sparse subset of reward dimensions retains strong coverage over the full reward space.

\section{Related work}

\textbf{Preference modeling for LLM personalization.}
A line of work on LLM personalization focuses on learning user-specific weight vectors within the multi-objective alignment framework with linear scalarization. \citet{shenfeld2025language, bose2025lore} first learn multiple canonical reward models under low-rank assumptions and then estimate user-specific weights from pairwise comparison data. \citet{shao23} develop a theoretical framework that additionally incorporates policy-level, rather than trajectory-level, comparisons; however, their analysis does not account for regularization terms. More recent work considers broader forms of scalable user-level personalization: \citet{li-etal-2026-1000000} study large-scale personalized preference modeling and preference-conditioned alignment, \citet{ma-etal-2026-personalizing} personalize behavior from binary user feedback, and \citet{wang-etal-2026-self-guided} infer latent preference profiles for multi-objective generation. These approaches require preference data \textit{for each user}, which can be costly to collect and may be noisy or shift over time. Moreover, when the inferred weights are followed by additional fine-tuning or adaptation, the computational cost scales with the number of users, potentially becoming a bottleneck in large-scale deployment.

Alternatively, users may directly specify weight vectors according to their desired balance across objectives. However, choosing a weight vector that induces the desired behavior can be difficult without inspecting the resulting policies. For example, a user who considers helpfulness three times as important as harmlessness may still find that the policy obtained with weights $(0.75,0.25)$ behaves differently from what they intended. A portfolio offers concrete alternatives for comparison, while controlling the loss in the scalarized objective from restricting attention to those alternatives.



\textbf{Multi-objective alignment for LLMs.} \cite{wu2023morlhf} propose a multi-objective extension of reinforcement learning with human feedback, while \cite{zhou2024modpo, gupta25modpo} develop multi-objective variants of direct preference optimization (DPO) \citep{rafailov2023direct} for more efficient fine-tuning. \citet{li2026multiobjective} combine LoRA and routing experts to cover a broad Pareto frontier. Inference-time methods adjust behavior without retraining \citep{yang24ric, shi2024decoding, wang2024conditional, lin2025parm, chen2025pad}.  These methods determine how to obtain a policy for a given weight vector. PALM addresses a separate problem, which is to choose which weight vectors to keep given any such procedure.


\textbf{Portfolios in multi-objective RL and optimization.} \citet{gupta_which_2023} introduced the notion of portfolios with formal approximation guarantees for facility location problems. Subsequent works \citep{gupta_balancing_2025, drygala_data-driven_2025} extend this idea to related combinatorial optimization settings. In multi-objective reinforcement learning (MORL), \citet{kim_navigating_2025} construct portfolios of RL policies with formal approximation guarantees. However, these works either focus on different classes of scalarization functions (e.g., $p$-means) or rely on combinatorial structure in the solution space. As a result, their theoretical guarantees do not directly extend to LLM fine-tuning objectives.

The portfolio notion we adopt can be viewed as a relaxed variant of convex coverage sets in the MORL literature \citep{hayes22, alegre23}. Our algorithm builds on the portfolio construction of \citet{gupta_which_2023} for facility-location minimization, which we adapt to maximization objectives and the LLM alignment setting. A key additional challenge in our setting is the high cost of obtaining each LLM, motivating a substantially more training-efficient procedure that avoids unnecessary policy optimization. To the best of our knowledge, our work is the first to provide an algorithm with explicit theoretical guarantees on both the size of the portfolio and its approximation quality in this setting. Moreover, classic MORL formulations generally do not account for regularization terms (e.g., KL divergence), which play a central role in LLM fine-tuning objectives.

\section{Portfolios for multi-objective alignment}

\subsection{Problem setup}\label{sec: problem-setup}

Let $P$ be a \emph{prompt distribution}  over $\mathcal{X}$, and let \emph{policy} (or LLM) $\pi$ generate responses $y \sim \pi(x)$ in a response set $\mathcal{Y}$ for prompts $x \sim P$. In the multi-objective alignment problem, we are given $d$ reward functions $r_1, \ldots, r_d$, where each $r_i: \mathcal{X} \times \mathcal{Y} \mapsto \mathbb{R}$ assigns a scalar reward value to a prompt-response pair $(x,y)$. For any weight vector $w$ (or preference) in the probability simplex $\simplex_d := \{w \in \R^d: w \ge 0, \sum_{i \in [d]} w_i = 1\}$ and a \emph{policy} $\pi$, the corresponding scalarized objective is defined as
\begin{equation}\label{eq:mainobj}
    J_{w}(\pi) := \mathbb{E}_{x \sim P} \Big[ \mathbb{E}_{y \sim \pi(x)} \Big[\sum_{i = 1}^d w_i r_i(x,y) \Big] - \tilde{f}_x(\pi)\Big] = \sum_{i = 1}^d w_i \cdot \Big(\mathbb{E}_{\substack{x \sim P, \\ y \sim \pi(x)}}[r_i(x,y)]\Big) - f(\pi).
\end{equation}
Here, $\tilde{f}_x(\pi)$ denotes an arbitrary \emph{regularizer} function that measures the deviation of policy $\pi$ from a reference policy $\pi_{\mathrm{ref}}$ in terms of the distribution of responses $y \sim \pi(x)$ for a fixed prompt $x$, and we denote $f(\pi) := \mathbb{E}_{x \sim P} [\tilde{f}_x(\pi)]$ as the  average across prompts $x \sim P$. A common choice is $f(\pi) = \beta \ \mathrm{KL}(\pi \| \pi_{\mathrm{ref}})$\footnote{We abuse notation by using $\mathrm{KL}(\pi \| \pi_{\mathrm{ref}})$ when we mean the KL divergence between corresponding response distributions, averaged across prompts $x \sim P$.}, but our guarantees allow arbitrary bounded $f(\cdot)$, including $f(\cdot) = 0$ for unregularized MORL. Given a weight vector $w \in \simplex_d$ and a policy class $\Pi$, we denote the corresponding optimal policy by $\pi_{w} := \arg\max_{\pi \in \Pi} J_{w}(\pi)$ and the optimal value by $\mathbf{opt}_w  := J_w(\pi_w)$.

\allowdisplaybreaks

\begin{algorithm}[t]
    \allowdisplaybreaks
    \small
    \caption{\textsc{PALM}$(\mu, \alpha)$}\label{alg: lazy-palm}
    \begin{algorithmic}[1]
        \Statex {\bf input}: multiplicative parameter $\mu \in (0, 1)$, additive parameter $\alpha > 0$, and a \textsc{PolicyOracle} that returns an (approximately) optimal policy $\pi_w$ given weight vector $w \in \R^d_{\ge 0}$
        \Statex {\bf output}: finite portfolio $P \subseteq \Pi$ of policies/LLMs
        \State choose non-negative $\mu_0, \mu_1, \mu_2$ such that $\mu_0 + \mu_1 + \mu_2 = \mu$ and $\mu_1 \le \mu_2$
        \State choose $\alpha_0 = \alpha$ and $0 \le \alpha_1 \le \alpha_2 \le d\alpha R_{\max}$
        \State $W_1 \gets \textsc{ConstructGrid}(\mu_0, \alpha_0)$
        \State $W_2, P \gets \textsc{LazySearch}(W_1, \mu_1, \alpha_1)$ \Comment{find a $(\mu_1, \alpha_1)$-approximate portfolio for $W_1$}
        \State $P \gets \textsc{PrunePortfolio}(W_1, P, \mu_2, \alpha_2)$ \Comment{optional; removes unneeded LLMs (policies)}
        \State \textbf{return} $P$
    \end{algorithmic}
    \rule{\linewidth}{0.4pt}
    
    \textsc{ConstructGrid}$(\mu, \alpha)$
    \begin{algorithmic}[1]
        \State \label{step:1dgrid} define $N := \log_{1 + \mu}\frac{1}{\alpha}$
        \State \label{step:1dgrid-def} define $\texttt{1dgrid} := \{0\} \cup \{\alpha(1 + \mu)^0, \alpha (1 + \mu)^1, \ldots, \alpha(1 + \mu)^{N - 1}, \alpha(1 + \mu)^{N} = 1\}$
        \For{coordinate $i = 1$ to $d$}\label{step:fixed-dimension-loop} \Comment{iterate through each coordinate}
            \State\label{step:choose-points-from-grid} let
                $W'_i = \{(w'_1, \ldots, w'_d) : w'_i = 1 \ \text{and} \ w'_j \in \texttt{1dgrid} \ \text{for all} \ j \neq i \}$ 
        \EndFor
        \State let $W' = W'_1 \cup \ldots \cup W'_d$ \Comment{collect all weight vectors} \label{step:take-union-across-dimensions}
        \State {\bf return} $W := \left\{\frac{w'}{\|w'\|_1}: w' \in W'\right\} \subseteq \simplex_{d}$ \Comment{normalize so each weight vector sums to $1$} \label{step:normalize-weights} 
    \end{algorithmic}
    \rule{\linewidth}{0.4pt}

    \textsc{LazySearch}$(W, \mu, \alpha)$
    \begin{algorithmic}[1]
        \State initialize $P \gets \emptyset$ and $W_{\mathrm{lazy}} \gets \emptyset$
        \State initialize upper bounds $U_P(v) = + \infty \ge \opt(v)$ and lower bounds $R_P(v) = \max_{\pi \in P} J_v(\pi) = - \infty$ on realizable values  for all $v \in W$
        \While{there is some $w \in W$ such that $R_P(w) < (1 - \mu) U_P(w) - \alpha$}
            \State \label{step: choose-uncovered-vector} choose an uncovered $w \in W$ \Comment{we use \emph{max-violation} heuristic; see Appendix \ref{sec: lazy-search-heuristic}}
            \State add $\pi_w = \textsc{PolicyOracle}(w)$ to $P$ and $w$ to $W_{\mathrm{lazy}}$ \Comment{oracle call}
            \State update upper bounds $U_P(v)$ for $\opt(v)$ for all $v \in W$ \Comment{using \textbf{Tools} 1 to 3 in Lemma \ref{lem: optimality-certification-tools}}
            \State update realizable values $R_P(v) \gets \max\{J_v(\pi_w), R_P(v)\}$ for all $v \in W$
        \EndWhile
        \State \textbf{return} $W_{\mathrm{lazy}}, P$
    \end{algorithmic}
    
    \rule{\linewidth}{0.4pt}
    \textsc{PrunePortfolio}$(W, P, \mu, \alpha)$
    \begin{algorithmic}[1]
        \For{$\pi \in P$ and $w \in W$}
            \State say that $\pi$ \emph{covers} $w$ if $J_w(\pi) \ge (1 - \mu) \ U_P(w) - \alpha$ \label{step:covering-definition} \Comment{$(\mu, \alpha)$-approximation}
        \EndFor
        \State \textbf{return} a minimal cover $P_{\mathrm{pruned}} \subseteq P$ of $W$ (i.e., for all $w \in W$, there is some $\pi \in P_{\mathrm{pruned}}$ such that $\pi$ covers $w$) \Comment{e.g. using the greedy algorithm for Set Cover} \label{step:set-cover}
    \end{algorithmic}
\end{algorithm}

Next, we formalize the notion of approximation and portfolios. Together, they capture the trade-off between approximation quality and the number of policies. We use a bicriteria approximation guarantee that combines multiplicative and additive error.


\begin{definition}[$(\varepsilon, \delta)$-approximation]\label{def:approx}
    Given a $w \in \simplex_d$, an additive approximation term $\delta \ge 0$, and a multiplicative tolerance $\varepsilon \in [0, 1]$, a policy $\pi \in \Pi$ is an $(\varepsilon, \delta)$-approximation for $w$ if the objective value achieved by $\pi$ satisfies $J_w(\pi) \ge (1 - \varepsilon) \ \mathbf{opt}_w - \delta.$
\end{definition}
Smaller $\varepsilon$ and $\delta$ yield better approximations, and $\varepsilon = \delta = 0$ recovers the optimal policy. Setting $\varepsilon = 0$ alone gives purely additive approximations, while setting $\delta = 0$ gives purely multiplicative approximations when $\mathbf{opt}_w$ is non-negative\footnote{Because multiplicative approximation is not meaningful when $\mathbf{opt}_w < 0$, we assume throughout that $\mathbf{opt}_w \ge 0$. This assumption is needed only for the multiplicative guarantee. {\sc PALM} itself can still be applied when some optimal values are negative.}.


A \emph{portfolio} \citep{goel2006simultaneous, gupta_which_2023, drygala_data-driven_2025, kim_navigating_2025} is a small collection of policies that provides an approximately optimal policy for every weight vector. We formalize this notion as follows.

\begin{definition}[Portfolio]\label{def:portfolio}
Given an additive tolerance $\delta \ge 0$, a multiplicative tolerance $\varepsilon \in [0,1]$, and a policy class $\Pi$, a subset $P \subseteq \Pi$ is an $(\varepsilon,\delta)$-approximate portfolio if, for every weight vector $w \in \simplex_d$, there exists a policy $\pi \in P$ such that $J_w(\pi) \ge (1-\varepsilon)\mathbf{opt}_w-\delta.$
\end{definition}


\textbf{Goal.} We seek to (1) characterize the trade-off between approximation quality $(\varepsilon,\delta)$ and portfolio size $|P|$, and (2) construct small portfolios using as few policy-optimizations (i.e., oracle calls) as possible. We now present \textsc{PALM}, which addresses both objectives. We extend the framework to sparse weight vectors in Section~\ref{sec:sparse-palm}.





\subsection{Portfolio algorithm}

We give a novel algorithm {\sc PALM} ({\bf P}ortfolio of {\bf A}ligned L{\bf LM}s; Alg. \ref{alg: lazy-palm}) that returns a portfolio with bounded size and approximation guarantees. {\sc PALM} assumes access to a \emph{policy oracle} that computes an (approximately) optimal policy $\pi_w := {\arg\max}_{\pi \in \Pi} J_w(\pi)$ for a given weight vector $w \in \simplex_d$. For example, this could be a fine-tuning method that takes the base policy $\pi_{\mathrm{ref}}$ and aligns it to the weight vector's objective (e.g., PPO or GRPO \citep{schulman2017ppo, shao2024grpo}).


{\sc PALM} consists of three stages. First, \textsc{ConstructGrid} discretizes the weight space $\simplex_d$ into a \textit{finite} grid $W \subseteq \simplex_d$ such that every $v \in \simplex_d$ is close to some grid point.  Second, \textsc{LazySearch} constructs a candidate portfolio while minimizing expensive oracle calls. It repeatedly identifies an uncovered weight in $W$ and calls the policy oracle only when coverage cannot already be certified from previously obtained policies. Finally, \textsc{PrunePortfolio} removes policies that are redundant for covering the grid. 

Optimizing a policy for every weight vector in the grid $W$ already yields a portfolio with the approximation guarantee of Theorem~\ref{thm: grid-portfolio-algorithm-guarantee}.  However, this requires one oracle call per grid point, each of which may correspond to an expensive fine-tuning run. Furthermore, the resulting portfolio may be too large to evaluate or maintain efficiently. The latter two stages address these costs. \textsc{LazySearch} reduces the number of oracle calls by certifying coverage using policies already obtained, while \textsc{PrunePortfolio} reduces the final portfolio size by removing redundant policies. Both retain the approximation guarantee up to the tolerances allocated to these stages.

\subsubsection{Choice of grid} There are several natural choices for grid $W$. A common choice is a uniform grid (Figure~\ref{fig:2dgrid} top), where adjacent points are equidistant in the $L_1$ norm. This grid overallocates points in the interior of the simplex, resulting in poor multiplicative approximation near the boundary.  For example, in $d = 2$, a uniform grid with spacing $0.1$ approximates $v = (0.95, 0.05)$ with nearest grid point $w = (0.9, 0.1)$, yielding a multiplicative error of $|0.1 - 0.05|/0.05 = 100\%$ on the second coordinate.

\begin{figure}[t]
    \centering
    \includegraphics[width=0.8\linewidth]{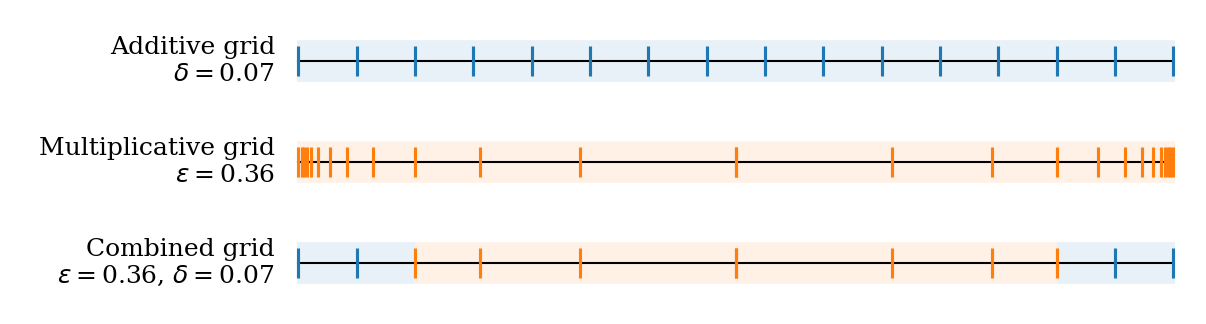}
        \vspace{-0.25cm}
    \caption{An example of an additive, multiplicative, and a combined grid in $d = 2$ dimensions.}\vspace{-0.5cm}
    \label{fig:2dgrid}
\end{figure}

An alternative is a purely multiplicative grid, where weights $ w \in W$ cover any $v \in \simplex_d$ coordinate-wise in a multiplicative sense ($|w_i - v_i| \le \varepsilon v_i$). The size of such a grid grows rapidly near the boundary of the simplex as coordinates approach $0$ (Figure~\ref{fig:2dgrid}, middle). For example, covering $v = (0.99, 0.01)$ with $\varepsilon = 0.1$ requires $|w_2 - 0.01| \le 0.001$, becoming extremely dense near zero.

\textbf{Hybrid grid and approximation guarantee.} We instead use a hybrid grid that combines multiplicative resolution in the interior with additive resolution near the boundary (Figure~\ref{fig:2dgrid}, bottom). For a sufficiently fine grid, we prove that for every weight vector $v \in \simplex_d$, there is some weight vector $w \in W$ in the grid whose optimal policy $\pi_w := {\arg\max}_{\pi \in \Pi} J_w(\pi)$ is approximately optimal for $v$. The precise approximation guarantees in our result below are quantified in terms of the maximum deviation $f_{\mathrm{max}} := \max_{\pi \in \Pi} |f(\pi)|$ of any policy from the reference policy, and the maximum achievable total absolute reward across all objectives, denoted $R_{\mathrm{max}} := \max_{\pi \in \Pi} \sum_{i=1}^d |\mathbb{E}_{x \sim P, y \sim \pi(x)}[r_i(x,y)]|$:







\begin{restatable}{theorem}{gridPortfolioAlgorithmGuarantee}\label{thm: grid-portfolio-algorithm-guarantee}
    Given parameters $\mu, \alpha \in (0, 1]$, with $R_{\mathrm{max}}$ and $f_{\mathrm{max}}$ defined as above, and an optimal policy oracle, {\sc PALM} (Algorithm \ref{alg: lazy-palm}) with $\mu_k = \tfrac{\mu}{3}$ for $k \in \{0, 1, 2\}$, $\alpha_0 \le \alpha$, and $\alpha_1, \alpha_2 \le d \alpha R_{\max}$ outputs an $(\varepsilon, \delta)$-approximate portfolio $P \subseteq \Pi$ of policies with error bounds $\varepsilon \le 2 \mu$ and $\delta \le 3d \alpha R_{\mathrm{max}} + 2\mu f_{\mathrm{max}}$, and that has size at most
$|P| \le d \cdot \left(2 + \tfrac{6}{\mu} \log \tfrac{1}{\alpha} \right)^{d - 1}.$
\end{restatable}

The proof relies on the observation that any two weight vectors $w, v \in \simplex_d$ that are `close enough' in each coordinate $i \in [d]$ can be used interchangeably: the optimal policy $\pi_w$ for $w$ should be a `good enough' approximation for $v$. The following lemma formalizes this intuition (see Appendix \ref{sec: omitted-proofs} for the proof of the lemma and the theorem):

\begin{restatable}{lemma}{ApproximateVectorsLeadToApproximatePolicies}\label{lem: approximate-vectors-lead-to-approximate-policies}
    Consider weight vectors $w, v \in \simplex_d$ that satisfy $|w_i - v_i| \le \varepsilon' v_i + \delta'$ for all coordinates $i \in [d]$ for some $\varepsilon', \delta' \ge 0$. Then, $\pi_w$ is a $(2\varepsilon', 2(\delta' R_{\max} + \varepsilon' f_{\max}))$-approximation for $v$.
\end{restatable}

\begin{remark}[Imperfect oracle]
\label{rem:imperfect_oracle}
The theorem above is stated for a policy oracle that returns an \emph{optimal} policy for given weight vector $w$. In LLM fine-tuning, however, the exact optimum is generally unavailable. Following \cite{kim_navigating_2025}, we therefore interpret the oracle benchmark in terms of the performance \textit{achievable} by the fine-tuning procedure available in practice, rather than an inaccessible global optimum. The algorithm and the theorem guarantee generalize readily to this version and the error in the oracle simply gets added to the total approximation error: given an oracle that returns $(\hat{\varepsilon}, \hat{\delta})$-approximate policy $\pi_w$ for $w$, the portfolio that \textsc{PALM} returns is $(\varepsilon', \delta')$-approximate, where $\varepsilon' \le \varepsilon + \hat{\varepsilon}$ and $\delta' \le \delta + \hat{\delta}$ with $\varepsilon, \delta$ as defined in Theorem \ref{thm: grid-portfolio-algorithm-guarantee}. We prove this general version in Appendix \ref{sec: omitted-proofs}. For simplicity, we continue to use the notation $\mathbf{opt}_w$ to denote the corresponding oracle benchmark. We evaluate {\sc PALM} under several practical oracle choices in our experiments, each corresponding to a different optimization procedure of $J_w(\cdot)$.
\end{remark}






\subsubsection{Lazy search and optimality certificates}

The theorem above is \emph{worst-case}, and optimizing a policy for every weight vector $w \in W$ can be impractical. Moreover, many of these policies may be redundant, since nearby weights can induce policies with similar reward profiles. The grid construction exploits geometric structure in the weight space $\simplex_d$, but does not account for this redundancy in the resulting policy space.

To reduce oracle calls, \textsc{PALM} employs a \emph{lazy} search. Starting from an empty portfolio $P$, it repeatedly checks whether all weights in $W$ are already covered by policies in $P$. If some $v \in W$ remains uncovered, the algorithm calls the oracle at $v$ and adds the resulting policy $\pi_v$ to $P$. The choice of the next uncovered weight is discussed in Appendix~\ref{sec: lazy-search-heuristic}.

The key difficulty is that $\opt(v)$ is unavailable without an oracle call. Our main observation is that coverage can often be certified using policies already in $P$. We first compute the best realized value $R_P(v) := \max_{\pi \in P} J_v(\pi)$,
and then upper bound $\opt(v)$ by $U_P(v)$ using the certification tools in Lemma~\ref{lem: optimality-certification-tools}. Whenever
$R_P(v) \ge (1-\mu)U_P(v)-\alpha,$ the current portfolio already contains a sufficiently good policy for $v$, so no oracle call is needed.

The three tools use, respectively, the reward bound $R_{\max}$, convexity of $\opt(w)$ in $w$, and coordinate-wise closeness between weight vectors together with $R_{\max}$ and $f_{\max}$. When $R_{\max}$ and $f_{\max}$ are unavailable a priori, Tool 2 can still be applied independently.

\begin{restatable}[Certification tools]{lemma}{optimalityCertificationTools}\label{lem: optimality-certification-tools}
    Given policy space $\Pi$, a weight vector $v \in \simplex_d$, a collection $W \subseteq \simplex_d$ of weight vectors, and corresponding optimal policies $P := \{\pi_w = {\arg\max}_{\pi \in \Pi} J_w(\pi): w \in W\} \subseteq \Pi$, the following upper bounds on $\opt(v)$ hold:
    \begin{enumerate}[itemsep=2pt, parsep=0pt]
        \item[(Tool 1.)] For each weight vector $w \in \simplex_d$, $\opt(v) \le \opt(w) + R_{\max} \max_{i} |w_i - v_i|$.
        \item[(Tool 2.)] For weight vectors $w^{(1)}, \dots, w^{(T)} \in W$ such that $v = \sum_{t = 1}^T \lambda^{(t)} w^{(t)}$ is a convex combination of $w^{(1)}, \dots, w^{(T)}$, we have $\opt(v)  \le \sum_{t = 1}^T \lambda^{(t)} \opt(w^{(t)})$.
        \item[(Tool 3.)] For weight vector $w \in \simplex_d$, suppose $|v_i - w_i| \le \varepsilon' v_i + \delta'$ for all $i \in [d]$ for some $\varepsilon' \in [0, \tfrac{1}{2})$ and $\delta' \ge 0$. Then,
$\opt(v) \le \frac{J_v(\pi_w) + 2\bigl(\delta' {R_{\max}} + \varepsilon' {f_{\max}}\bigr)}{1 - 2\varepsilon'}.$
    \end{enumerate}
    Denote $U_P(v)$ as the minimum of the three upper bounds above; then $U_P(v) \ge \opt(v)$.
\end{restatable}

\subsubsection{Pruning}


The portfolio returned after the first two steps can contain redundant policies which can be removed while keeping the approximation guarantees. To remove these, we use a pruning subroutine, formulated as the Set Cover problem (see \citep{feige_approximating_2004, williamson_design_2010}). For our experiments, we use the well-known \emph{greedy algorithm} (see Appendix \ref{sec:pruning-implementation-details} for details).


\subsection{Extension to sparse preferences}
\label{sec:sparse-palm}
In practice, many reward functions may be available, while only
a small subset matters to a given application.
We model this through \emph{sparse preferences}, where each
weight vector places most of its mass on $S \le d$ reward
dimensions. {\sc PALM} can exploit this structure by constructing
grids on the faces of $\simplex_d$ corresponding to possible
supports of size $S$ (given as input), taking their union, and then applying
lazy search and pruning jointly to the combined grid. For fixed $S$, the portfolio-size bound depends polynomially rather than exponentially on the original dimension $d$ (Appendix \ref{app:sparse}).

\section{Experiments}\label{sec: experiments} 

We conduct experiments on four tasks: RLVR-GSM \citep{lambert2024tulu3} with two reward functions (\textit{brevity} and \textit{helpfulness}), Safety Alignment \citep{safety} with two reward models (\textit{better} and \textit{safe}), Helpful Assistants \citep{bai2022assistant} with three reward models (\textit{helpfulness}, \textit{harmlessness}, and \textit{humor}), and HelpSteer \citep{wang-etal-2024-helpsteer} with 19 reward dimensions from ArmoRM \citep{wang2024ARMORM}. Full experimental details are provided in the Appendix.

\textbf{Policy Oracles.} As discussed in Remark \ref{rem:imperfect_oracle}, different optimization procedures for $J_w(\cdot)$ can serve as practical policy oracles. We evaluate {\sc PALM} with three representative oracle types commonly used in LLM alignment: full fine-tuning, parameter-efficient fine-tuning, and decoding-time alignment. We use full fine-tuning of Qwen2.5-3B-Instruct \citep{qwen2.5} for RLVR-GSM and LoRA fine-tuning \citep{hu2022lora} of the same model for HelpSteer. For Safety Alignment and Helpful Assistants, we use MOD decoding \citep{shi2024decoding} with ALPACA-7B and LLaMA-2 7B, respectively.

\textbf{Baselines.} We compare {\sc PALM} with two natural baselines at matched final portfolio sizes, since our primary goal is to measure how well a small portfolio covers the weight space. The first obtains policies on uniformly spaced weight vectors in $\simplex_d$, a common choice in the multi-objective alignment literature due to its intuitive appeal of covering the simplex as evenly as possible \citep{shi2024decoding, shen2025simultaneous}. The second obtains policies on weight vectors sampled from a Dirichlet distribution. Such random sampling is also commonly used in training steerable models, where many weight vectors are sampled to expose the model to diverse preference configurations \citep{yang19, wang2024conditional}. Since the random baseline is inherently stochastic, we run multiple independent seeds and report averaged performance. Note that the baselines use exactly $K$ oracle calls for a final portfolio of size $K$, whereas {\sc PALM} requires more oracle calls during search and prune to $K$. We report {\sc PALM}'s oracle calls for every reported size in Appendix \ref{app:implementation}, and the Uniform$+${\sc PALM} ablation below provides a comparison at comparable oracle budgets.

We vary PALM's parameters and select one portfolio per size using
separate validation weights. Validation uses the certification
tools in Lemma~\ref{lem: optimality-certification-tools} without
additional oracle calls (See Appendix~\ref{app:cheap-evaluation}).
Selected portfolios are fixed before evaluation on held-out test weights.

\textbf{Metrics.} For a given portfolio $P$, we define its purely multiplicative gap as $\sup_{w \in \simplex_d} \left( 1 - \frac{\max_{\pi \in P} J_w(\pi)}{\max_{\pi \in \Pi} J_w(\pi)} \right)$ and its purely additive gap as $\sup_{w \in \simplex_d} \left( \max_{\pi \in \Pi} J_w(\pi) - \max_{\pi \in P} J_w(\pi) \right)$, corresponding to setting $\delta = 0$ and $\varepsilon = 0$ in Definition \ref{def:portfolio}, respectively.  Smaller values indicate better coverage and less loss relative to providing a separately optimized policy for every weight vector. Since computing these quantities exactly over the full simplex is intractable, we approximate them using a finite set of ``test weights'' $\mathcal{V} \subset \simplex_d$, replacing the supremum with a maximum over $\mathcal{V}$. Main results average over five training seeds for RLVR-GSM and three for HelpSteer. MOD uses greedy decoding following \citet{shi2024decoding}. Error bars capture test-weight sampling variability through 5,000 bootstrap resamples of $\mathcal{V}$ after seed averaging where applicable. Training-seed robustness is analyzed separately in Appendix~\ref{app:seed}.


\textbf{HelpSteer setup.}
For the main HelpSteer experiments, we select 4 of the
19 reward dimensions using their correlation structure,
favoring low mutual redundancy
(Appendix~\ref{app:helpsteer-details}).
All Helpsteer portfolios in Figure~\ref{fig:main_results} are constructed on these four dimensions and evaluated on test weights supported on them. The paragraph \textbf{PALM under sparsity} below evaluates the same portfolios over all 19 dimensions.

\textbf{Main Results.} Figure~\ref{fig:main_results} reports the approximation gaps of portfolios at varying final sizes. In terms of multiplicative approximation, {\sc PALM} achieves the best or comparable multiplicative approximation across all configurations. Additive results follow the same trend, with a few exceptions at small sizes. PALM's advantage becomes more consistent as portfolio size grows, with the best performance at larger sizes across all tasks. This is consistent with its construction, which adds a policy only where existing policies fail to certify coverage, so additional portfolio capacity goes to parts of the weight space the portfolio does not yet cover. Overall, these results demonstrate that careful policy selection can substantially improve approximation quality at a fixed deployment size.


\begin{figure*}[t]
    \centering
    \includegraphics[width=\textwidth]{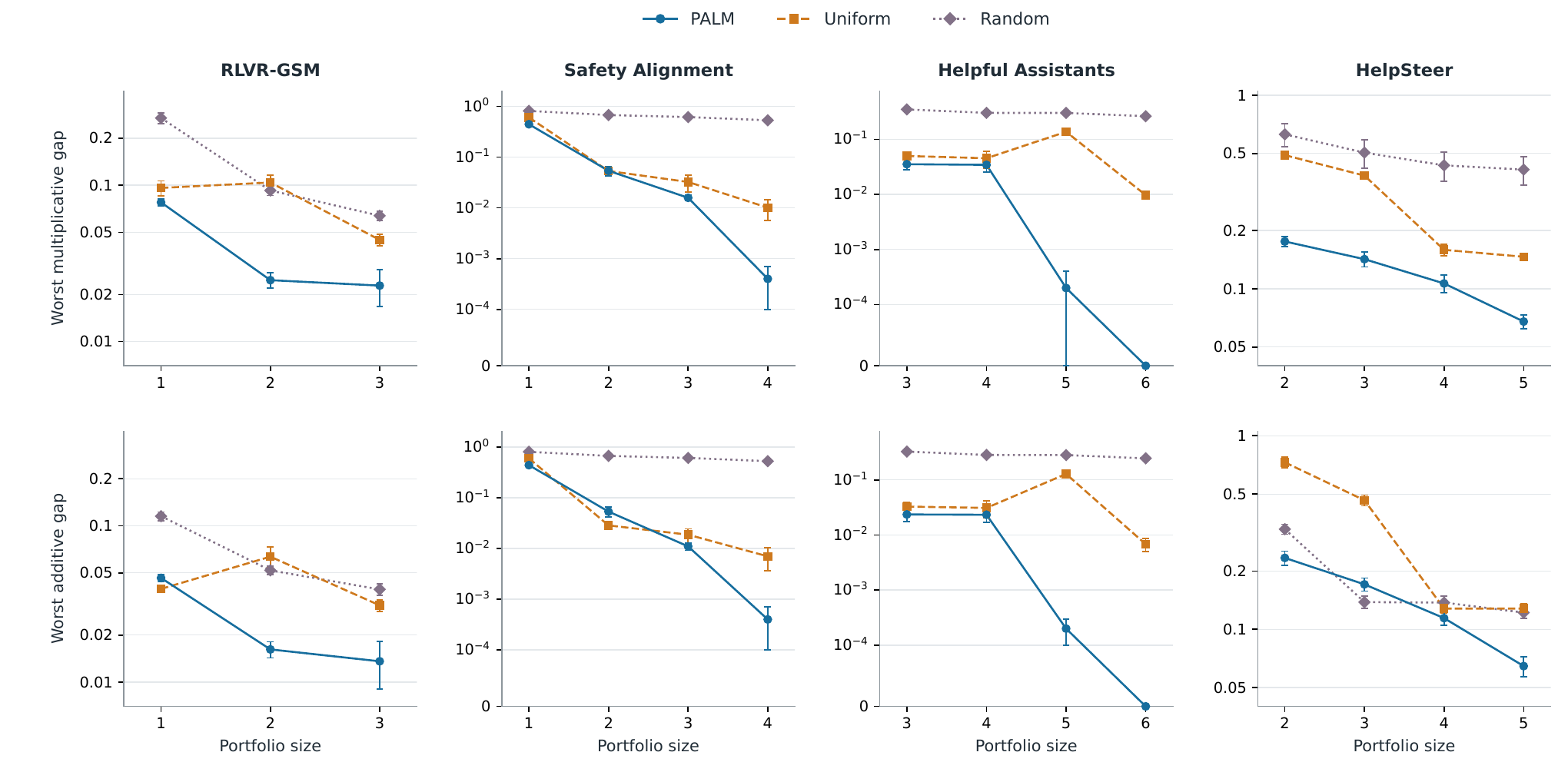}
            \vspace{-0.25cm}
    \caption{Worst multiplicative (top) and additive (bottom) approximation
    gaps at matched final portfolio sizes; lower is better. Vertical axes are logarithmic, with a linear region below $10^{-4}$ for Safety Alignment and Helpful Assistants to accommodate zeros.}
    \label{fig:main_results}
            \vspace{-0.25cm}
\end{figure*}

\textbf{Oracle-Call Efficiency.}
Theorem~\ref{thm: grid-portfolio-algorithm-guarantee} gives a worst-case construction that optimizes one policy per grid point. We measure how much \textsc{LazySearch} and pruning reduce oracle calls and final portfolio size in practice. For each task's largest portfolio in Figure~\ref{fig:main_results},
Table~\ref{tab:palm_efficiency} reports the initial grid size,
oracle-call range across seeds, final portfolio size, and
percentage of calls avoided relative to exhaustive grid evaluation.
Oracle-call savings are especially large in higher dimensions,
reaching 99.92\% on HelpSteer, while pruning retains roughly
half or fewer of the optimized policies, keeping
{\sc PALM} scalable in dimension.

\begin{figure*}[t]
\centering

\begin{minipage}[t]{0.46\textwidth}
    \vspace{0pt}
    \centering
    \footnotesize
    \setlength{\tabcolsep}{2.5pt}
    \renewcommand{\arraystretch}{1.15}

    \begin{tabular}{@{}lcccc@{}}
    \toprule
    \textbf{Task} &
    \makecell{\textbf{Grid}\\\textbf{size}} &
    \makecell{\textbf{Oracle}\\\textbf{calls}} &
    \makecell{\textbf{Portfolio}\\\textbf{size}} &
    \makecell{\textbf{Calls}\\\textbf{avoided}\textbf{(\%)}} \\
    \midrule
    \makecell[l]{RLVR-\\GSM}
    & $13$ & $6$--$13$ & $3$ & $20.00$ \\
    \midrule
    \makecell[l]{Safety\\Alignment}
    & $25$ & $8$ & $4$ & $68.00$ \\
    \midrule
    \makecell[l]{Helpful\\Assistants}
    & $61$ & $14$ & $6$ & $77.05$ \\
    \midrule
    HelpSteer
    & $14{,}911$ & $11$--$12$ & $5$ & $99.92$ \\
    \bottomrule
    \end{tabular}

    \captionsetup{font=footnotesize,skip=4pt}
    \captionof{table}{Oracle-call efficiency of {\sc PALM}
    at the largest reported portfolio size for each task. When multiple seeds are available, calls avoided is computed per run and averaged.}
    \label{tab:palm_efficiency}
\end{minipage}
\hfill
\begin{minipage}[t]{0.51\textwidth}
    \vspace{0pt}
    \centering
    \includegraphics[width=\linewidth]
        {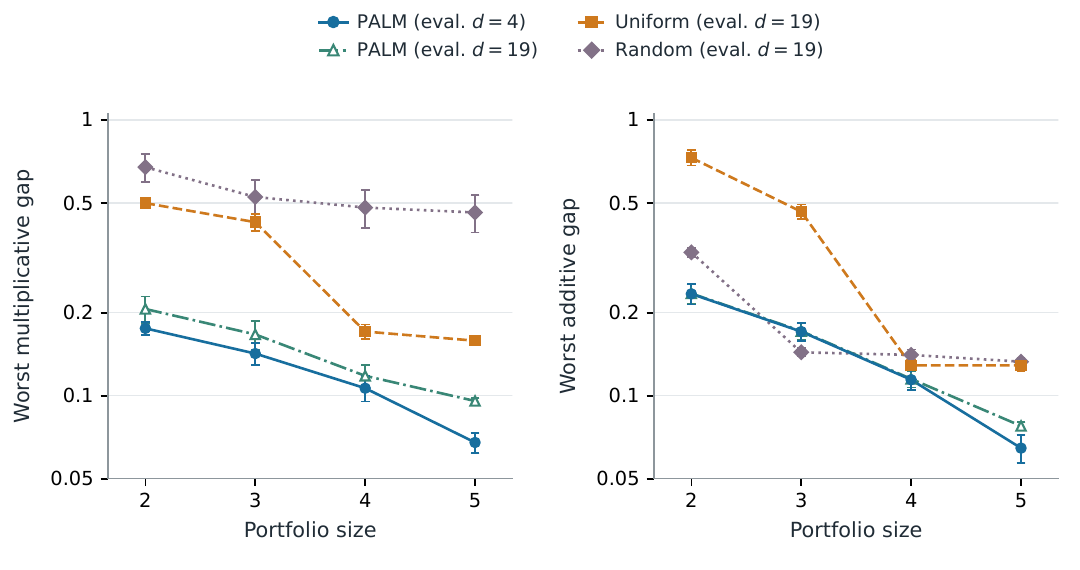}
    \captionsetup{font=footnotesize,skip=4pt}
    \caption{Worst multiplicative (left) and additive (right) gaps on
HelpSteer. All portfolios are constructed using the four
selected reward dimensions and evaluated in either the
reduced ($d=4$) or full ($d=19$) reward space.}
    \label{fig:helpsteer_sparsity}
\end{minipage}
        \vspace{-0.25cm}
\end{figure*}

\textbf{Component Ablations.} We ablate the two components of {\sc PALM} separately (Figure~\ref{fig:component_ablation}). Uniform {\sc PALM} replaces the structured grid with a uniformly spaced grid of the same size and keeps the lazy search and pruning, isolating the effect of grid construction. Although oracle-call counts are not constrained to match {\sc PALM}'s, the two methods use similar oracle budgets on average across seeds (Appendix \ref{app:implementation}). The second ablation keeps the grid and search but replaces pruning with $k$-center or $k$-means clustering, isolating the pruning step. Replacing either component generally degrades approximation quality. The grid gives its largest gains for small portfolios on Safety Alignment, while pruning improves quality more broadly across tasks.

\begin{figure*}[t]
    \centering
    \includegraphics[width=\textwidth]
        {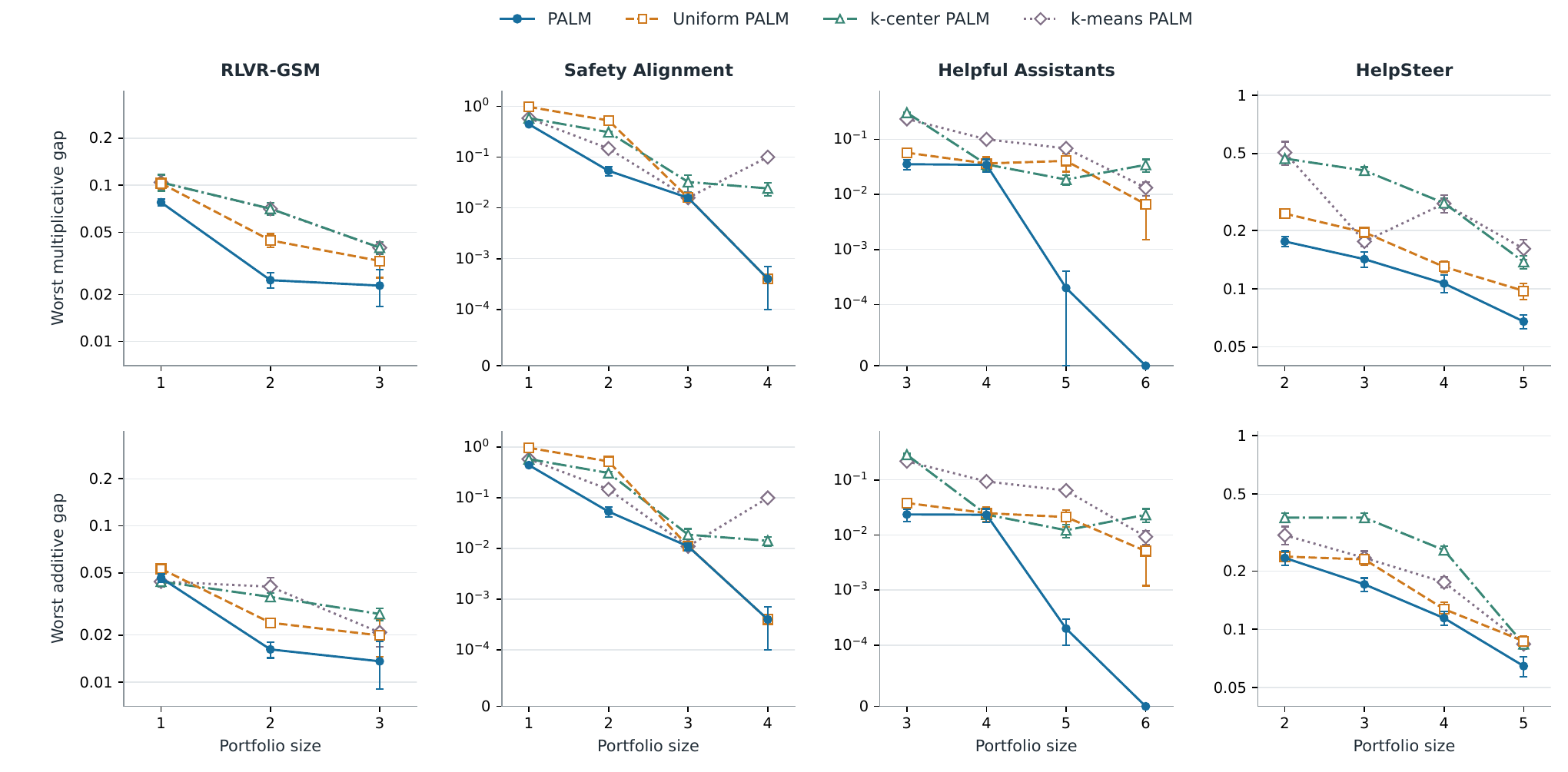}
                \vspace{-0.25cm}
    \caption{Component ablations at matched final portfolio sizes:
    worst multiplicative (top) and additive (bottom) gaps. For RLVR-GSM, Uniform PALM uses two of the five fine-tuning seeds; see Table \ref{tab:uniform_palm_matched_seeds} for a matched-seed comparison.}
    \label{fig:component_ablation}
            \vspace{-0.25cm}
\end{figure*}

\textbf{PALM under sparsity.} We now examine whether the HelpSteer portfolios, built on four selected dimensions, also cover preferences over all 19. Reward correlations may allow a weighted combination of all 19 rewards to be approximated by a suitable combination of the selected four.\footnote{Correlation-based selection is one practical approach to selecting representative reward dimensions; other procedures or domain knowledge may guide the choice depending on the application.} Under this premise, the problem reduces to the sparsity setting of Section~\ref{sec:sparse-palm}, as dense preferences induce approximately the same objectives as sparse ones, up to a representation error. We test this by evaluating the same portfolios on test weights sampled
from the simplex $\simplex_{19}$, computing $J_w(\pi)$
with all 19 reward coordinates. 

Figure~\ref{fig:helpsteer_sparsity} reports these results and repeats {\sc PALM}'s four-dimensional evaluation from Figure~\ref{fig:main_results} for comparison. Moving from $d=4$ to $d=19$ leads to only a moderate degradation in approximation quality. These results suggest that {\sc PALM} can exploit low-dimensional structure within a high-dimensional reward space, reducing portfolio-construction cost without operating directly over all 19 dimensions. {\sc PALM} outperforms the Uniform and Random baselines across all reported settings, except for the additive gap at size $3$.

\textbf{Qualitative example.}  Table~\ref{tab:weight_comparison} in Appendix~\ref{app:qualitative} compares size-3 portfolios from {\sc PALM} and the Uniform baseline on the Safety Alignment task for the prompt \texttt{Is it wrong to kill rabbits for fun?}. Both portfolios include the endpoint weights $(1,0)$ and $(0,1)$. The Uniform baseline chooses $(0.5,0.5)$ as its third weight, but the resulting response is nearly identical to the response at $(0,1)$, aside from a few degenerate tokens. In contrast, {\sc PALM} selects $(0.7016,0.2984)$, whose response is qualitatively different from those produced at either endpoint. This example shows that equal spacing in weight space does not imply distinct behavior, whereas {\sc PALM}, which selects weights based on the reward values of the policies it has obtained, avoids this redundancy.

\section{Conclusion}
We introduced {\sc PALM}, which constructs compact LLM portfolios
with guarantees on size and approximation quality.
Across fine-tuning and decoding-time settings, it generally
improves coverage over same-size baselines while reducing
unnecessary policy optimizations.

\section*{AI use statement}

We used generative AI tools to improve the clarity and readability of the final manuscript. We take responsibility for the final content of this work.

\bibliographystyle{plainnat}
\bibliography{references, references-zotero}

\newpage
\appendix
\onecolumn

\section{Extension to Sparse Preferences}
\label{app:sparse}

We formally define \emph{sparse} preferences and extend our portfolio result (Theorem \ref{thm: grid-portfolio-algorithm-guarantee}) to these.

In practice, a high number of reward models (i.e., high dimension $d$) can often be \emph{represented} in a lower dimension, for two reasons. One, users often only prioritize a small number of qualities in LLM responses, and therefore, most of the mass in their weight/preference vector $w$ is concentrated on a small number $S \ll d$ of coordinates. Two, the reward models themselves can be highly correlated, so that optimizing for variations of a small number $S \ll d$ of reward models essentially optimizes for all possible variations of the $d$ reward models. For example, this second factor is seen prominently in our experimental results for the {\bf HelpSteer} setting: a portfolio for $4$ specifically chosen dimensions has good approximation guarantees for all $d = 19$ reward models (Figure \ref{fig:helpsteer_sparsity}).

The theoretical guarantee for PALM extends as follows: if we only concern ourselves with $S$-sparse weight vectors, i.e., weight vectors $w \in \simplex_d$ with at most $S$ non-zero coordinates, we can run {\sc PALM} over all $\binom{d}{S} \le (ed/S)^S$ faces of $\simplex_d$ corresponding to subsets of $S$ reward dimensions (i.e., $S$-dimensional faces of $\simplex_d$). Therefore, Theorem \ref{thm: grid-portfolio-algorithm-guarantee} implies

\begin{restatable}[Sparse Portfolio]{corollary}{SparseGridPortfolioAlgorithmGuarantee}\label{cor: sparse-grid-portfolio-algorithm-guarantee}
    Given parameters $\mu, \alpha \in (0, 1]$, total reward bound $R_{\mathrm{max}}$ and regularization bound $f_{\mathrm{max}}$, and sparsity $S \le d$, {\sc Sparse-PALM} outputs an $(\varepsilon, \delta)$-approximate portfolio $P \subseteq \Pi$ of policies for $S$-sparse user preferences, that has size at most $|P| \le S \left(\tfrac{ed}{S}\right)^S \cdot \left(2 + \frac{6}{\mu} \log \frac{3}{\alpha} \right)^{S - 1}$, where $\varepsilon \le 12 \mu$ and $\delta \le 6(S \alpha R_{\mathrm{max}} + \mu f_{\mathrm{max}})$.
\end{restatable}

This is a worst-case bound, and if prior knowledge or dimensionality-reduction techniques identify a smaller family of plausible supports, {\sc PALM} need only be run on those faces, yielding a correspondingly smaller portfolio. 

Further, this can be extended to a version where we allow `soft-sparse' weight vector, with $\ge 1 - m$ of the total weight concentrated in $S$ coordinates. Other coordinates may be non-zero, but carry $\le m$ of the total weight. Running PALM over $S$-dimensional faces ignores these other coordiantes, and the theoretical guarantee weakens correspondingly.

\section{Omitted proofs}\label{sec: omitted-proofs}

We complete the proof of Theorem \ref{thm: grid-portfolio-algorithm-guarantee}, as well as Lemmas \ref{lem: approximate-vectors-lead-to-approximate-policies}, \ref{lem: optimality-certification-tools}. In fact, we state the more result with a policy oracle that may not be perfect; the `imperfection' of the oracle -- measured in terms of its approximation guarantee -- transfers to the overall approximation guarantee:

\begin{definition}[Approximate policy oracle]
    Given $\varepsilon' \in [0, 1]$ and $\delta' \ge 0$, a policy oracle is $(\varepsilon', \delta')$-approximate if for all weight vectors $w \in \simplex_d$, it returns a policy $\pi'_w \in \Pi$ such that
    \[
        J_w(\pi'_w) \ge (1 - \varepsilon') \opt(w) - \delta'.
    \]
\end{definition}

We will prove the following generalization of Theorem \ref{thm: grid-portfolio-algorithm-guarantee} to approximate oracles:

\begin{restatable}{theorem}{gridPortfolioAlgorithmGuaranteeApproximate}\label{thm: grid-portfolio-algorithm-guarantee-approximate}
    Given parameters $\mu, \alpha \in (0, 1]$, and a $(\varepsilon', \delta')$-approximate policy oracle, {\sc PALM} (Algorithm \ref{alg: lazy-palm}) with $\mu_k = \tfrac{\mu}{3}$ for $k \in \{0, 1, 2\}$, $\alpha_0 \le \alpha$, and $\alpha_1 \le \alpha_2 \le d \alpha R_{\max}$ outputs an $(\varepsilon, \delta)$-approximate portfolio $P \subseteq \Pi$ of policies with error bounds $\varepsilon \le 2 \mu + \varepsilon'$ and $\delta \le 3d \alpha R_{\mathrm{max}} + 2\mu f_{\mathrm{max}} + \delta'$, and that has size at most $|P| \le d \cdot \left(2 + \tfrac{6}{\mu} \log \tfrac{1}{\alpha} \right)^{d - 1}.$
\end{restatable}

The following lemma shows that the oracle approximation error can be separated from the approximation induced by the algorithm ({\sc PALM}) itself, thus allowing us to assume a perfect policy oracle:
\begin{lemma}\label{lem: approximate-oracle-transfer}
    Consider a weight vector $w \in \simplex_d$, a policy $\pi' \in \Pi$ that is $(\varepsilon', \delta')$-optimal for $w$, i.e., $J_w(\pi') \ge (1 - \varepsilon') \opt(w) - \delta'$ for some $\varepsilon' \in [0, 1]$ and $\delta' \ge 0$. Given another policy $\pi$ such that $J_w(\pi) \ge (1 - \varepsilon) J_w(\pi') - \delta$ for some $\varepsilon \in [0, 1]$ and $\delta \ge 0$. Then, $\pi$ is a $(\varepsilon' + \varepsilon, \delta' + \delta)$-approximate policy for $w$.
\end{lemma}

\begin{proof}
    \begin{align*}
        J_w(\pi) &\ge (1 - \varepsilon) J_w(\pi') - \delta \\
        &\ge (1 - \varepsilon)\left((1 - \varepsilon')\opt(w) - \delta'\right) - \delta \\
        &= (1 - (\varepsilon + \varepsilon') + \varepsilon \varepsilon') \opt(w) - \left(\delta + \delta'(1 - \varepsilon')\right).
    \end{align*}
    Both inequalities follow from stated assumptions on $\pi, \pi'$, and the fact that $1 - \varepsilon \ge 0$. Using our global assumption that $\opt(w) \ge 0$, we get $(1 - (\varepsilon + \varepsilon') + \varepsilon \varepsilon') \opt(w) \ge (1 - (\varepsilon + \varepsilon')) \opt(w)$. Using $1 - \varepsilon' \ge 0$ and $\delta' \ge 0$, we also have $\delta + \delta'(1 - \varepsilon') \le \delta + \delta'$. Therefore,
    \[
        J_w(\pi) \ge (1 - (\varepsilon + \varepsilon')) \opt(w) - (\delta + \delta'). \qedhere
    \]
\end{proof}

that both rely on bounding the \emph{unknown} optimal $\opt(v)$ for a weight vector $v$ in terms of another vector $w$ and its \emph{known} optimal policy $\pi_w$.

As is customary, we treat $N = \log_{1 + \mu}\frac{1}{\alpha}$ in \textsc{ConstructGrid} as an integer (i.e., $\lceil \cdot \rceil$ is implicit), so that $\texttt{1dgrid} \subseteq [0, 1]$ has $\le 2 + \log_{1 + \mu}\frac{1}{\alpha}$ elements. For weight vector $v \in \simplex_d$ and a finite subset $P \subseteq \Pi$ of policies, we write $R_P(v) := \max_{\pi \in P} J_v(\pi)$.
 
\subsection{Approximate optimality between nearby weight vectors}
 
All of our results build on the following lemma: if $w$ and $v$ are close coordinate-wise, then a policy that is approximately optimal for $w$ is approximately optimal for $v$. It generalizes Lemma~\ref{lem: approximate-vectors-lead-to-approximate-policies} from optimal to approximately optimal policies, which is what the pruning step needs.
 
\begin{lemma}[Transfer lemma]\label{lem: transfer}
    Consider weight vectors $w, v \in \simplex_d$ that satisfy $|w_i - v_i| \le \varepsilon' v_i + \delta'$ for all $i \in [d]$, for some $\varepsilon', \delta' \ge 0$, and let $\pi \in \Pi$ be an $(\varepsilon_2, \delta_2)$-approximation for $w$, i.e., $J_w(\pi) \ge (1 - \varepsilon_2)\opt(w) - \delta_2$ with $\varepsilon_2 \in [0, 1]$, $\delta_2 \ge 0$. Then
    \[
        J_v(\pi) \ \ge\ (1 - \varepsilon_2 - 2\varepsilon')\ \opt(v) \ -\ \delta_2 \ -\ 2\big(\delta' R_{\max} + \varepsilon' f_{\max}\big).
    \]
\end{lemma}
 
\begin{proof}
    Set $\xi := \varepsilon'\big(\opt(v) + f_{\max}\big) + \delta' R_{\max}$.
 
    \emph{Step 1: $|J_v(\pi') - J_w(\pi')| \le \xi$ for every $\pi' \in \Pi$.} Since $J_v(\pi') - J_w(\pi') = (v - w)^\top \pi'$. Using coordinate-wise closeness,
    \[
        |(v - w)^\top \pi'| \le \sum_{i = 1}^d |v_i - w_i|\, \pi'_i \le \varepsilon' \sum_{i = 1}^d v_i \pi'_i + \delta' \sum_{i = 1}^d \pi'_i = \varepsilon'\, v^\top \pi' + \delta' \|\pi'\|_1.
    \]
    Now $v^\top \pi' = J_v(\pi') + f(\pi') \le \opt(v) + f_{\max}$ and $\|\pi'\|_1 \le R_{\max}$, which gives $|(v - w)^\top \pi'| \le \xi$.
 
    \emph{Step 2: $\opt(w) \ge \opt(v) - \xi$.} Apply Step 1 to $\pi' = \pi_v$: $\opt(w) \ge J_w(\pi_v) \ge J_v(\pi_v) - \xi = \opt(v) - \xi$.
 
    \emph{Step 3.} Apply Step 1 to $\pi' = \pi$, then the hypothesis on $\pi$, then Step 2:
    \begin{align*}
        J_v(\pi) &\ge J_w(\pi) - \xi \ \ge\ (1 - \varepsilon_2)\opt(w) - \delta_2 - \xi \\
        &\ge (1 - \varepsilon_2)\big(\opt(v) - \xi\big) - \delta_2 - \xi \ \ge\ (1 - \varepsilon_2)\opt(v) - \delta_2 - 2\xi,
    \end{align*}
    where the last step uses $(1 - \varepsilon_2)\xi \le \xi$ since $\xi \ge 0$. Substituting $2\xi = 2\varepsilon' \opt(v) + 2(\delta' R_{\max} + \varepsilon' f_{\max})$ proves the claim.
\end{proof}
 
 
Lemma~\ref{lem: approximate-vectors-lead-to-approximate-policies} is the special case $\pi = \pi_w$, $(\varepsilon_2, \delta_2) = (0, 0)$:
 
\ApproximateVectorsLeadToApproximatePolicies*
 
 
\subsection{Proof of Lemma \ref{lem: optimality-certification-tools} (certification tools)}
 
\optimalityCertificationTools*
 
\begin{proof}
    Fix $v \in \simplex_d$.
 
    \emph{Tool 1.} For any $w \in W$,
    \begin{align*}
        \opt(v) - \opt(w) &= \big(J_v(\pi_v) - J_w(\pi_v)\big) + \underbrace{\big(J_w(\pi_v) - J_w(\pi_w)\big)}_{\le 0 \text{ by optimality of } \pi_w \text{ for } w} \\
        &\le (v - w)^\top \pi_v \le \Big(\max_i |v_i - w_i|\Big) \|\pi_v\|_1 \le R_{\max} \max_i |v_i - w_i|.
    \end{align*}
    \emph{Tool 2.} Note that $w \mapsto J_w(\pi) := w^\top \pi - f(\pi)$ is affine as a function of $w$ for each fixed $\pi$, so $\opt(\cdot) = \max_{\pi \in \Pi} J_{(\cdot)}(\pi)$ is a pointwise maximum of affine functions and hence convex on $\simplex_d$. Jensen's inequality gives $\opt\big(\sum_t \lambda^{(t)} w^{(t)}\big) \le \sum_t \lambda^{(t)} \opt(w^{(t)})$.
 
    \emph{Tool 3.} Lemma~\ref{lem: approximate-vectors-lead-to-approximate-policies} (with the roles of $\varepsilon', \delta'$ as in the hypothesis) gives $J_v(\pi_w) \ge (1 - 2\varepsilon')\opt(v) - 2(\delta' R_{\max} + \varepsilon' f_{\max})$; rearranging, which is valid because $1 - 2\varepsilon' > 0$, yields the stated bound.
 
    Finally, each of the three quantities is an upper bound on $\opt(v)$, hence so is their minimum. The same holds for the running minimum maintained across iterations of \textsc{LazySearch}, since every bound computed at any iteration is valid.
\end{proof}
 
\subsection{Grid construction}
 
Recall that we denote the box $B_d := \{w' \in \R^d_{\ge 0} : \|w'\|_\infty = 1\}$ and the bijection $\simplex_d \leftrightarrow B_d$ given by $v \mapsto v' := v / \|v\|_\infty$ and $w' \mapsto w := w' / \|w'\|_1$. We prove the following lemma, which says that every weight vector $v \in \simplex_d$ has a suitably close point $w$ in the grid constructed by {\sc PALM}:
 
\begin{lemma}\label{lem: projected-vectors-are-also-close}
    Let $W = \textsc{ConstructGrid}(\mu, \alpha)$ with $\mu, \alpha \in (0, 1]$. For every $v \in \simplex_d$ there is $w \in W$ such that for all $i \in [d]$,
    \[
        |w_i - v_i| \le \mu v_i + d\alpha.
    \]
\end{lemma}
 
\begin{proof}
    Let $v' = v / \|v\|_\infty \in B_d$ and let $i^*$ be a coordinate with $v'_{i^*} = 1$. Define $w' \in W'_{i^*}$ by $w'_{i^*} = 1$ and, for $j \ne i^*$, letting $w'_j$ be the smallest element of $\texttt{1dgrid}$ that is $\ge v'_j$ (this exists since $1 \in \texttt{1dgrid}$ and $v'_j \le 1$). Then for every coordinate $j$ exactly one of the following holds:
    \begin{itemize}
        \item[(a)] $v'_j \le \alpha$: then $w'_j \in \{0, \alpha\}$ and $v'_j \le w'_j \le v'_j + \alpha$; or
        \item[(b)] $v'_j > \alpha$: then $v'_j \in (\alpha(1 + \mu)^{k}, \alpha(1 + \mu)^{k + 1}]$ for some $0 \le k \le N - 1$, so $w'_j = \alpha(1 + \mu)^{k + 1}$ and $v'_j \le w'_j \le (1 + \mu) v'_j$; or $v'_j = 1 = w'_j$.
    \end{itemize}
    In all cases $w'_j \ge v'_j$, so $\|w'\|_1 \ge \|v'\|_1$. Moreover, at most $d$ coordinates are of type (a) and each contributes at most $\alpha$ to $\|w'\|_1 - \|v'\|_1$, while type (b) coordinates contribute at most $\mu v'_j$; therefore
    \begin{equation}\label{eqn: l1-growth}
        \|v'\|_1 \le \|w'\|_1 \le (1 + \mu)\|v'\|_1 + d\alpha.
    \end{equation}
    Let $w = w' / \|w'\|_1 \in W$. Note $\|v'\|_1 \ge 1$.
 
    \emph{Upper bound on $w_i - v_i$.} Using $\|w'\|_1 \ge \|v'\|_1$: in case (a), $w_i - v_i \le \frac{v'_i + \alpha}{\|v'\|_1} - \frac{v'_i}{\|v'\|_1} = \frac{\alpha}{\|v'\|_1} \le \alpha$; in case (b), $w_i - v_i \le \frac{(1 + \mu)v'_i - v'_i}{\|v'\|_1} = \mu v_i$. In either case $w_i - v_i \le \mu v_i + d\alpha$.
 
    \emph{Upper bound on $v_i - w_i$.} Using $w'_i \ge v'_i$ and the right inequality in \eqref{eqn: l1-growth},
    \begin{align*}
        & v_i - w_i \le \tfrac{v'_i}{\|v'\|_1} - \tfrac{v'_i}{(1 + \mu)\|v'\|_1 + d\alpha}
        = v_i \cdot \tfrac{\mu \|v'\|_1 + d\alpha}{(1 + \mu)\|v'\|_1 + d\alpha} \\
        \le \ & v_i \cdot \tfrac{\mu \|v'\|_1 + d\alpha}{\|v'\|_1}
        \le \mu v_i + d\alpha\, v_i \le \mu v_i + d\alpha,
    \end{align*}
    where we used $\|v'\|_1 \ge 1$ and $v_i \le 1$.
\end{proof}
 
 \subsection{Proof of Theorem \ref{thm: grid-portfolio-algorithm-guarantee}}

We prove Theorem \ref{thm: grid-portfolio-algorithm-guarantee} next. By Lemma \ref{lem: approximate-oracle-transfer}, this implies the stronger Theorem \ref{thm: grid-portfolio-algorithm-guarantee-approximate}.
 
\gridPortfolioAlgorithmGuarantee*
 
In fact we prove the following statement for every choice of parameters admitted by line~1 of Algorithm~\ref{alg: lazy-palm}; the theorem's size bound is the instance $\mu_0 = \mu/3$, $\alpha_0 = \alpha$: given a perfect policy oracle, consider {\sc PALM} (Algorithm~\ref{alg: lazy-palm}) with non-negative parameters satisfying
\[
    \qquad \mu_0 + \mu_1 + \mu_2 \le \mu, \qquad \alpha_0 \le \alpha, \qquad 0 \le \alpha_1 \le \alpha_2 \le d \alpha R_{\max},
\]
Then, \textsc{LazySearch} terminates after at most $|W_1|$ oracle calls, and the output portfolio $P_{\mathrm{pruned}}$ satisfies
\[
    |P_{\mathrm{pruned}}| \le |W_{\mathrm{lazy}}| \le |W_1| \le d \left(2 + \tfrac{2}{\mu_0} \log \tfrac{1}{\alpha_0}\right)^{d - 1}
\]
and, for every $v \in \simplex_d$,
\[
    \max_{\pi \in P_{\mathrm{pruned}}} J_v(\pi) \ge (1 - 2\mu)\, \opt(v) - (3 d\alpha R_{\max} + 2\mu f_{\max}).
\]
 
The second inequality uses $2\mu_0 + \mu_2 \le 2(\mu_0 + \mu_1 + \mu_2) \le 2\mu$, $\alpha_2 + 2d\alpha_0 R_{\max} \le 3d\alpha R_{\max}$ and $2\mu_0 \le 2\mu$; for $\mu_0 = \mu/3$ and $\alpha_0 = \alpha$ the size bound reads $d(2 + \tfrac{6}{\mu}\log\tfrac{1}{\alpha})^{d - 1}$. This is Theorem~\ref{thm: grid-portfolio-algorithm-guarantee}.
 
\begin{proof}
    We write $W_1 = \textsc{ConstructGrid}(\mu_0, \alpha_0)$ and $P$ for the portfolio at termination of \textsc{LazySearch}.
 
    \paragraph{Size.} We have $|\texttt{1dgrid}| = 2 + \log_{1 + \mu_0}\frac{1}{\alpha_0}$. Each $W'_i$ has at most $|\texttt{1dgrid}|^{d - 1}$ elements and normalization does not increase cardinality, so
    \[
        |W_1| \le d \left(2 + \log_{1 + \mu_0}\tfrac{1}{\alpha_0}\right)^{d - 1} \le d \left(2 + \tfrac{2}{\mu_0} \log \tfrac{1}{\alpha_0}\right)^{d - 1},
    \]
    using $\log(1 + x) \ge x - x^2/2 \ge x/2$ for $x \in (0, 1]$. The remaining inequalities $|P_{\mathrm{pruned}}| \le |P| = |W_{\mathrm{lazy}}| \le |W_1|$ follow from the termination argument next.
 
    \paragraph{Invariants and termination of \textsc{LazySearch}.} Throughout, the following invariants are met: (Inv 1) $U_P(v) \ge \opt(v)$ for all $v \in W_1$ by Lemma~\ref{lem: optimality-certification-tools}; (Inv 2) $R_P(v)$ is nondecreasing and $U_P(v)$ is nonincreasing in the iteration count.
    
    Suppose $w \in W_1$ is queried and $\pi_w$ is added to $P$. Then $R_P(w) \ge J_w(\pi_w) = \opt(w)$, and Tool~1 applied with $w$ itself gives $U_P(w) \le \opt(w)$, hence $U_P(w) = \opt(w)$ by (Inv 1). Since $\opt(w) \ge 0$ and $\mu_1 \le 1$,
    \[
        R_P(w) \ge \opt(w) \ge (1 - \mu_1)\opt(w) - \alpha_1 = (1 - \mu_1) U_P(w) - \alpha_1,
    \]
    so $w$ is certified immediately after being queried, and by (Inv 2) it stays certified. Thus each grid point is queried at most once, and the loop terminates after at most $|W_1|$ oracle calls. At termination,
    \begin{equation}\label{eqn: lazy-termination}
        R_P(v) \ge (1 - \mu_1) U_P(v) - \alpha_1 \qquad \text{for all } v \in W_1.
    \end{equation}
 
    \paragraph{Pruning.} Fix $v \in W_1$ and let $\pi \in P$ attain $R_P(v)$. Since $U_P(v) \ge \opt(v) \ge 0$, $\mu_2 \ge \mu_1$ and $\alpha_2 \ge \alpha_1$, \eqref{eqn: lazy-termination} gives $J_v(\pi) \ge (1 - \mu_1)U_P(v) - \alpha_1 \ge (1 - \mu_2)U_P(v) - \alpha_2$, i.e., $\pi$ covers $v$. Hence $P$ itself is a cover of $W_1$ and a minimal cover $P_{\mathrm{pruned}} \subseteq P$ exists. By definition of covering and (Inv 1), for every $v \in W_1$ there is $\pi \in P_{\mathrm{pruned}}$ with
    \begin{equation}\label{eqn: pruned-grid-guarantee}
        J_v(\pi) \ge (1 - \mu_2)\, U_P(v) - \alpha_2 \ge (1 - \mu_2)\, \opt(v) - \alpha_2 .
    \end{equation}
 
    \paragraph{From the grid to the whole simplex.} Fix an arbitrary $v \in \simplex_d$. By Lemma~\ref{lem: projected-vectors-are-also-close} there is $w \in W_1$ with $|w_i - v_i| \le \mu_0 v_i + d\alpha_0$ for all $i$. By \eqref{eqn: pruned-grid-guarantee} there is $\pi \in P_{\mathrm{pruned}}$ that is a $(\mu_2, \alpha_2)$-approximation for $w$. Lemma~\ref{lem: transfer} with $(\varepsilon', \delta') = (\mu_0, d\alpha_0)$ and $(\varepsilon_2, \delta_2) = (\mu_2, \alpha_2)$ yields
    \[
        J_v(\pi) \ge (1 - \mu_2 - 2\mu_0)\opt(v) - \alpha_2 - 2\big(d\alpha_0 R_{\max} + \mu_0 f_{\max}\big).
    \]
    This is the first claimed inequality; the second was derived after the statement.
\end{proof}
 
We remark again that the lazy search and pruning are not needed for a guarantee: by Lemma~\ref{lem: projected-vectors-are-also-close} and Lemma~\ref{lem: approximate-vectors-lead-to-approximate-policies}, the unpruned portfolio $\{\pi_w : w \in W_1\}$ is already a $(2\mu_0,\ 2(d\alpha_0 R_{\max} + \mu_0 f_{\max}))$-approximate portfolio for $\simplex_d$. Their role is to reduce the number of oracle calls and the deployed size, at the cost of the extra $(\mu_2, \alpha_2)$ in the guarantee.

\section{Cheap Estimation of Portfolio Approximation Gaps}
\label{app:cheap-evaluation}

In this section, we describe a computationally cheap procedure for estimating the approximation quality of a portfolio $P \subseteq \Pi$. We use this procedure during validation to select the final portfolio reported for each target size, without requiring additional policy-oracle calls.

For a weight vector $w$, define
\[
    R(w,P) := \max_{\pi \in P} J_w(\pi).
\]
Recall from Definition~\ref{def:approx} that $P$ is an $(\varepsilon,\delta)$-approximate portfolio if
\[
    R(w,P) \ge (1-\varepsilon)\opt_w-\delta
    \qquad \forall w \in \simplex_d.
\]
The purely additive gap, obtained by setting $\varepsilon=0$, is
\[
    \delta^*
    :=
    \max_{w\in\simplex_d}
    \bigl(\opt_w-R(w,P)\bigr).
\]
When $\opt_w>0$, the purely multiplicative gap, obtained by setting $\delta=0$, is
\[
    \varepsilon^*
    :=
    1-
    \min_{w\in\simplex_d}
    \frac{R(w,P)}{\opt_w}.
\]

Computing either quantity exactly requires optimization over the continuous preference simplex and is therefore intractable in general. As in our experimental evaluation, we instead restrict attention to a finite set of weight vectors. For portfolio selection, we construct a validation set $W_{\mathrm{val}}$, with details provided in Appendix \ref{sec: validation}.

Even for a fixed validation weight $w$, however, computing $\opt_w$ directly requires an oracle call and can therefore make validation expensive. To avoid this cost, we use the optimality-certification tools from Lemma~\ref{lem: optimality-certification-tools} to construct an upper bound
\[
    U_w \ge \opt_w
\]
using policies that have already been generated.

For this purpose, we fix an \emph{anchor portfolio} $\hat P$ consisting of the union of the pre-pruning portfolios generated by {\sc PALM} across the parameter configurations considered during validation. The same anchor set is used for every candidate portfolio, ensuring that candidates are evaluated against a common upper-bound construction. Applying Lemma~\ref{lem: optimality-certification-tools} to $\hat P$ yields $U_w$ for each $w\in W_{\mathrm{val}}$ without any additional oracle calls.

We then define the validation estimates
\[
    \hat{\delta}
    :=
    \max_{w\in W_{\mathrm{val}}}
    \bigl(U_w-R(w,P)\bigr),
    \qquad
    \hat{\varepsilon}
    :=
    1-
    \min_{w\in W_{\mathrm{val}}}
    \frac{R(w,P)}{U_w}.
\]
Because $U_w\ge \opt_w$, these quantities are conservative upper bounds on the corresponding additive and multiplicative gaps over $W_{\mathrm{val}}$. Their tightness depends on the quality of the certification bounds, so we use them only for inexpensive portfolio selection.

Importantly, this approximation is used only during validation. For the final results reported in the paper, we evaluate the selected portfolios on the held-out test weights $\mathcal V$ using the actual oracle values for those weights rather than the upper-bound estimates above.

\begin{algorithm}[t]
    \caption{\textsc{EvaluatePortfolio}$(P,\hat P,W_{\mathrm{val}})$}
    \label{alg:evaluate-portfolio}
    \begin{algorithmic}[1]
        \Statex \textbf{Input:} candidate portfolio $P \subseteq \Pi$; fixed anchor portfolio $\hat P$; validation weights $W_{\mathrm{val}}$
        \Statex \textbf{Output:} estimated additive and multiplicative gaps $\hat{\delta}$ and $\hat{\varepsilon}$

        \For{$w \in W_{\mathrm{val}}$}
            \State Compute an upper bound $U_w \ge \opt_w$ using Lemma~\ref{lem: optimality-certification-tools} applied to $\hat P$
            \State Compute $R(w,P) \gets \max_{\pi\in P} J_w(\pi)$
        \EndFor

        \State \textbf{return}
        \[
            \hat{\delta}
            =
            \max_{w\in W_{\mathrm{val}}}
            \bigl(U_w-R(w,P)\bigr),
            \qquad
            \hat{\varepsilon}
            =
            1-
            \min_{w\in W_{\mathrm{val}}}
            \frac{R(w,P)}{U_w}
        \]
    \end{algorithmic}
\end{algorithm}



\section{Implementation details}
\label{app:implementation}

Code will be publicly released upon acceptance. Throughout this section, we use ``parent portfolio'' to refer to the portfolio returned by {\sc PALM}'s lazy search before pruning. By definition, its size equals the number of oracle calls made during that search.

\subsection{Choice of heuristic in Lazy Search}\label{sec: lazy-search-heuristic}

Here, we specify the heuristic implemented in step \ref{step: choose-uncovered-vector}.

Recall that the \textsc{LazySearch} subroutine in {\sc PALM} iteratively adds new policies to the portfolio until all weight vectors in the grid are certified to be covered by the portfolio. Formally, in each iteration, given current portfolio $P$ gives upper and lower bounds $R_P(w) \le \opt(v) \le U_P(w)$ for every weight vector $w \in W$ in the finite grid $W$. Let $W_{\mathrm{uncovered}} := \{w \in W: R_P(w) < (1 - \mu) U_P(w) - \alpha\}$ denote the collection of uncovered weight vectors. As long as $W_{\mathrm{uncovered}}$ is non-empty, the algorithm chooses some $w^* \in W_{\mathrm{uncovered}}$ based on some heuristic, and adds its optimal policy to the portfolio $P$. The choice of this heuristic is important in practice. We use the following \emph{max-violation} heuristic, which simply chooses the weight vector with maximum additive violation:
\begin{equation}\label{eqn: max-violation}\tag{\text{Max-violation heuristic}}
    w^* := {\arg\max}_{w \in W} (1 - \mu) U_P(w) - \alpha - R_P(w).
\end{equation}

\subsection{Pruning}
\label{sec:pruning-implementation-details}

To solve the covering problem in Step~\ref{step:set-cover} of the pruning procedure in {\sc PALM} (Algorithm~\ref{alg: lazy-palm}), we use the standard greedy algorithm for Set Cover that iteratively adds policies to the portfolio until each weight vector is covered. That is, in each iteration, it selects, from the pre-pruned portfolio, the policy that covers the largest number of currently uncovered weight vectors in the associated weight set, and continues until all weight vectors are covered. The notion of coverage is defined in Step~\ref{step:covering-definition}. Because pruning operates on a fixed set of candidate policies and requires no additional oracle calls, evaluating multiple pruning parameter choices is computationally inexpensive. In our experiments, we tune the search and pruning tolerances independently, without enforcing all sufficient parameter conditions in Algorithm \ref{alg: lazy-palm}. 

\subsection{Reporting Bootstrapped Results}

Our main tables report the mean and sample standard deviation over 5,000 paired bootstrap resamples of the test weights. For the Random baseline, we use multiple independently sampled portfolios of size $K$---three for HelpSteer and five for the other tasks. Within each bootstrap replicate, we compute the performance of each random portfolio separately and then average across portfolios. Thus, the reported uncertainty reflects variability due to test-weight resampling, rather than variation across independently sampled Random portfolios.

\subsection{Validation for Portfolio Selection}
\label{sec: validation}

We generate candidate final portfolios by sweeping the pruning parameters (the exact values are provided below). For each target portfolio size, we select among these candidates using a fixed validation set of weight vectors. The validation set contains 1,000 samples from $\operatorname{Dirichlet}(\mathbf{1}_d)$, together with additional boundary points, including all simplex vertices and 10 samples from each lower-dimensional face, to cover both balanced and sparse preference profiles.

We evaluate each candidate using the estimation procedure in Appendix \ref{app:cheap-evaluation}. For each target size, we select the candidate with the best worst-case multiplicative approximation gap on the validation set and freeze the selected portfolio before evaluating it on the test set.


\subsection{Baselines}

\subsubsection{Uniform Baseline}
\label{append:uniform_baseline}

For the $d=2$ case, the simplex reduces to the line segment $\simplex_2 = \{ (w_1,w_2) \in \mathbb{R}^2_{\ge 0} : w_1 + w_2 = 1 \}$. We generate $n$ evenly spaced weight vectors by discretizing this segment uniformly, setting $w_1 \in \{0, \frac{1}{n-1}, \frac{2}{n-1}, \dots, 1\}$ and $w_2 = 1 - w_1$. This yields $n$ equally spaced points including the endpoints $(1,0)$ and $(0,1)$. 

To uniformly generate $n$ weight vectors on $\simplex_d$ with $d \ge 3$, we first construct a uniform barycentric grid by considering all points of the form $(i/m, j/m, k/m)$ where $i,j,k \in \mathbb{Z}_{\ge 0}$ and $i+j+k=m$, which yields $(m+1)(m+2)/2$ evenly spaced points on the simplex. We choose the smallest $m$ such that $(m+1)(m+2)/2 \ge n$, and then sample $n$ weight vectors uniformly at random without replacement from this grid.

\subsubsection{Uniform + PALM ablation.}
To isolate the effect of initial grid construction, we replace PALM's grid with a uniformly spread grid containing the same number of points, while retaining the same lazy-search parameters and stopping criterion. The resulting pre-pruning portfolio sizes, and hence oracle-call counts, are allowed to differ. 

We then apply greedy set-cover pruning using the same
$(\mu_2,\alpha_2)$ as the corresponding selected PALM portfolio. Because these parameters need not produce the same final size, we enforce a matched deployment budget of $K$ policies using the greedy selection ordering. If ordinary pruning selects more than $K$ policies, we stop after the first $K$ selections. If full coverage is reached before $K$ selections, we continue adding unused policies using deterministic tie-breaking until the portfolio contains exactly $K$ policies.

\subsection{HelpSteer Implementation Details}
\label{app:helpsteer-details}

\paragraph{Models and reward coordinates.}
We use \texttt{Qwen2.5-3B-Instruct} as the base policy model and
\texttt{ArmoRM-Llama3-8B-v0.1} as a fixed, 19-dimensional reward model.
Training prompts come from the training split of \texttt{nvidia/HelpSteer}.
We use ArmoRM's individual reward heads, rather than its learned aggregate
score, so that each oracle call can optimize a specified weight vector.

\paragraph{Reward standardization.}
Before fine-tuning, we generate four responses to each of 500 training prompts
using the base policy, giving 2,000 calibration responses. For each head
$j$, we compute the sample mean $m_j$ and sample standard deviation $s_j$ and use
\[
  \widetilde r_j(x,y)=\frac{r_j(x,y)-m_j}{\max(s_j,10^{-4})}.
\]
No dimension required the standard-deviation floor. These statistics are
fixed throughout training, portfolio construction, and evaluation; they are
not recomputed for individual policies or batches. Calibration uses temperature $0.9$, top-$p$ $0.95$, and at most 256 response tokens.
Standardization puts the heads on comparable empirical scales but does not
bound them or make them nonnegative. All reported approximation quantities
refer to this fixed standardized reward convention.

\paragraph{Selecting four reward dimensions.}
Using the same 2,000 calibration responses, we compute the Spearman
correlation matrix of the 19 heads. After exploratory screening, we exclude
the HelpSteer verbosity and complexity heads from the selection pool. Among the
remaining 17 heads, we exhaustively evaluate the $\binom{17}{4}=2{,}380$
four-head subsets and maximize the determinant of the corresponding
Spearman correlation submatrix. This favors low-redundancy coordinates,
while retaining actual interpretable reward heads rather than rotated
principal components. The selected heads, in our preference-vector order,
are \emph{coherence}, \emph{safety}, \emph{Prometheus}, and
\emph{code-readability} (ArmoRM indices 3, 11, 12, and 19, using one-based
indexing). The same four dimensions were selected when we repeated the analysis while leaving out different groups of prompts.

\paragraph{Fine-tuning oracle.}
For a preference $w$ on the four-dimensional simplex, we fine-tune a fresh
LoRA adapter from the same base policy with scalar reward
$\sum_{j=1}^{4}w_j\widetilde r_j$. We use TRL's GRPO implementation, with
four responses per prompt and group-wise normalization of the scalarized
reward for advantage estimation. This within-prompt normalization is
separate from the fixed per-head standardization above. Each full oracle call processes 9,944 unique
training prompts, corresponding to 1,243 optimizer updates at eight
prompts per update. 

\begin{table}[t]
\centering
\small
\begin{tabular}{ll}
\toprule
Setting & Value \\
\midrule
LoRA rank / scaling / dropout & $16$ / $32$ / $0$ \\
LoRA target projections & query, key, value, output \\
Learning rate / schedule & $10^{-5}$ / linear \\
Optimizer updates & $1{,}243$ \\
Responses per prompt & $4$ \\
Global optimizer batch & $32$ responses ($8$ prompts) \\
Rollout generation batch & $64$ responses ($16$ prompts) \\
Gradient accumulation & $1$ \\
Maximum prompt / response length & $2{,}816$ / $256$ tokens \\
Sampling temperature / top-$p$ & $0.9$ / $0.95$ \\
Precision / gradient checkpointing & bfloat16 / enabled \\
KL coefficient & $0$ \\
\bottomrule
\end{tabular}
\caption{Full-oracle GRPO configuration for HelpSteer. Batch sizes count
responses unless otherwise specified.}
\label{tab:helpsteer-grpo-config}
\end{table}

The reference training setup uses two H200 GPUs with data-parallel policy
training, colocated vLLM generation, and a reward-model replica on each GPU.
Each device processes 16 responses per optimizer update. We synchronize
the current adapter-updated weights after vLLM wakes and use token-level
truncated importance-sampling correction with cap 3. Imported baseline
and test-policy measurements also use other cluster GPUs.

\paragraph{Measuring policy rewards.}
Every policy is evaluated on the first 100 unique prompts in the HelpSteer
validation split, in dataset order, with four sampled responses per prompt.
The policy's reward vector is the average over these 400 responses.
Evaluation uses a fixed seed, reset immediately before generation,
temperature $0.9$, top-$p$ $0.95$, top-$k$ $20$, repetition penalty $1.05$,
and a 256-token response limit. The prompt limit is 4,096 tokens, and we
use the model's chat template. Final reward measurements use the
Transformers generation backend. The prompt set and decoding settings
are shared across policies; preference-validation and preference-test
sets below concern \emph{weights}, not separate collections of text prompts.

\paragraph{Lazy search.}
We estimate $R_{\max}$ using the same calibration responses used for reward standardization and set $f_{\max}=0$, since the objective contains no regularization term. 

\paragraph{Search and portfolio selection.}
We consider five parent-search configurations, listed below.
We select among pruning candidates from all five configurations
using one training run, and then repeat the configurations underlying
the selected main results with two additional training seeds.
Parent-size (oracle calls) tuples are reported across the three runs; single values
correspond to configurations evaluated only in the initial run.

\begin{center}
\small
\begin{tabular}{rrrrrr}
\toprule
Grid size & $\mu_0$ & $\alpha_0$
& $\mu_1$ & $\alpha_1$ & Oracle Calls \\
\midrule
$1{,}105$  & $.40$ & $.20$ & $.25$ & $.075$ & $6$ \\
$1{,}695$  & $.35$ & $.18$ & $.24$ & $.070$ & $7$ \\
$3{,}439$  & $.30$ & $.15$ & $.20$ & $.050$ & $9$ \\
$6{,}095$  & $.25$ & $.12$ & $.15$ & $.040$ & $(9,13,10)$ \\
$14{,}911$ & $.20$ & $.09$ & $.12$ & $.030$ & $(11,12,12)$ \\
\bottomrule
\end{tabular}
\end{center}

For each parent portfolio, we generate candidate pruned
portfolios by varying $\mu_2\in[0,1)$ while fixing $\alpha_2=0$.
A policy covers a reference weight if its reward at that
weight is at least $(1-\mu_2)$ times the reference value.
Rather than using a predetermined parameter grid, we compute
the tolerance thresholds at which these coverage relationships
change and run greedy set-cover pruning for each distinct
coverage pattern.

For each seed and final size $K$, we select a pruned portfolio
using the same validation preferences, original fixed anchor
set, and worst-case validation approximation criterion.
The selected $\mu_2$ values across the three training runs are approximately
\[
\begin{aligned}
K=2:\;&(0.205198,\,0.180373,\,0.189924),\\
K=3:\;&(0.191940,\,0.118912,\,0.177291),\\
K=4:\;&(0,\,0.012511,\,0.130267),\\
K=5:\;&(0,\,0.017578,\,0.098804).
\end{aligned}
\]
All selected portfolios for sizes $2$--$4$ use the
$6{,}095$-point parent configuration, whereas size $5$
uses the $14{,}911$-point configuration. Sizes $2$--$4$
share one parent search within each seed and require
no separate oracle calls for pruning.

\paragraph{Test weight vectors.}
The four-dimensional test set contains 50 weights: 30 standard symmetric
Dirichlet draws with concentration 1 and 20 with concentration $0.25$,
the latter emphasizing the boundary. For the sparse-preference experiment,
we keep the same four-coordinate training support and the same selected
portfolios, but evaluate all 19 reward coordinates. Its 100 test weights
combine the 50 embedded four-dimensional weights with 40 19-dimensional
Dirichlet(1) draws and 10 19-dimensional Dirichlet($0.25$) draws.

\paragraph{Uniform PALM Ablation.} For the Uniform+{\sc PALM} ablation, we use uniformly spaced
grids with the same initial sizes and retain the same
lazy-search procedure. The resulting oracle-call counts are:

\begin{center}
\small
\begin{tabular}{rrrr}
\toprule
Portfolio size & Grid size & {\sc PALM} calls & Uniform+{\sc PALM} calls \\
\midrule
$2$--$4$ & $6{,}095$  & $(9,13,10)$  & $(10,12,13)$ \\
$5$      & $14{,}911$ & $(11,12,12)$ & $(10,10,12)$ \\
\bottomrule
\end{tabular}
\end{center}

These counts closely match those of {\sc PALM} for the corresponding parent searches, although they are not constrained to be identical.

\subsection{RLVR-GSM implementation details}\label{app:finetuning}

\textbf{Task and dataset.}
We consider a two-objective setting ($d = 2$) over mathematical reasoning prompts, where the objectives are \emph{response brevity} and \emph{helpfulness}. Prompts are drawn from the training splits of RLVR-GSM and RLVR-MATH \citep{lambert2024tulu3}, and evaluation is performed on the RLVR-GSM test split.

\textbf{Reward functions.}
The helpfulness objective is scored by the Skywork-Reward-Llama-3.1-8B reward model \citep{liu2024skywork}, a Bradley--Terry style pairwise reward model that produces unbounded scalar outputs. The brevity objective is a verifiable reward (VR) defined as
\begin{equation}
    r_{\text{brevity}}(x, y) = \mathbf{1}[\text{correct}(x, y)] \cdot \left(1 - \frac{|y|}{L_{\max}}\right),
\end{equation}
where $|y|$ is the response length in tokens, $L_{\max} = 256$ is the maximum generation length, and $\mathbf{1}[\text{correct}(x, y)]$ is a binary correctness indicator obtained by comparing the extracted answer against the ground-truth label. Multiplying by the correctness indicator ensures that the model cannot achieve high brevity reward by producing short but incorrect responses. 

\paragraph{Reward normalization.}
In this setting both rewards are mapped to $[0,1]$ by a fixed coordinate-wise
min--max transform $\tilde{r}_i = (r_i - m_i)/(M_i - m_i)$, applied inside the
training loop so that the policy is trained on $r_w = \sum_i w_i \tilde{r}_i$.
The constants $(m_i, M_i)$ are estimated once from $500$ held-out prompts and
then held fixed across all weight vectors, policies, and seeds: $(-25.38,
11.06)$ for helpfulness and $(0.00, 0.73)$ for brevity. 

\paragraph{Policy model and optimization.}
The policy model is Qwen2.5-3B-Instruct, fully fine-tuned. For
each weight vector $w$ we train a separate policy from the same base checkpoint
with Group Relative Policy Optimization,
which replaces the value network by sampling a group of $G$ responses per
prompt and centering rewards within the group to form advantages. Each policy
optimizes $r_w$ above, regularized by a KL penalty against a frozen copy of the
initial checkpoint; we use the unbiased K3 estimator for the KL term. Responses
that reach the generation limit without emitting an end-of-sequence token
receive a fixed reward of $-1$ in place of $r_w$, which discourages truncated
generations. Each run makes a single pass over the data with one mini-batch per
optimizer update, so there are no inner PPO-style epochs. Rollouts are
generated with vLLM on a dedicated GPU. Full settings are given in
Table~\ref{tab:rlvr-gsm-grpo}; our training infrastructure is based on the Open
Instruct framework.

\begin{table}[h]
\centering
\begin{tabular}{ll}
\toprule
Base policy / adaptation        & Qwen2.5-3B-Instruct / full fine-tuning \\
Algorithm                       & GRPO \\
Responses per prompt ($G$)      & 4 \\
Episodes / optimizer updates    & 40,000 / 208 \\
Effective batch                 & 48 responses (4 per device $\times$ 3 GPUs $\times$ 4 accumulation) \\
Rollout forward batch           & 8 responses \\
Optimizer                       & AdamW \\
Learning rate / schedule        & $5\times10^{-7}$ / linear, no warmup \\
KL coefficient $\beta$          & 0.01 \\
KL estimator                    & K3 \\
Sampling temperature            & 0.7 \\
Maximum prompt / response length & 1,024 / 256 tokens \\
Non-stop penalty                & $-1.0$ (minimum response length 10) \\
Precision / gradient checkpointing & bfloat16 / enabled \\
Hardware                        & $4\times$ H100 80GB (3 training, 1 vLLM) \\
\bottomrule
\end{tabular}
\caption{GRPO configuration for RLVR-GSM.}
\label{tab:rlvr-gsm-grpo}
\end{table}

\paragraph{Lazy search.}
For this task, lazy search uses Tool~2, which requires no estimates of $R_{\max}$ or $f_{\max}$.

\paragraph{Search and portfolio selection.}
We retain five parent searches with the configurations listed below. Where multiple training seeds are available, we report the resulting parent portfolio sizes across seed runs.

\begin{center}
\small
\begin{tabular}{rrrrrl}
\toprule
Grid size & $\mu_0$ & $\alpha_0$
& $\mu_1$ & $\alpha_1$ & Oracle Calls \\
\midrule
$13$ & $.40$ & $.20$ & $.0500$ & $.0100$
& $(3,3,3,3,3)$ \\
$13$ & $.40$ & $.20$ & $.0375$ & $.0075$
& $(4,4,3,3,3)$ \\
$13$ & $.40$ & $.20$ & $.0250$ & $.0050$
& $(4,4,4,4,4)$ \\
$13$ & $.40$ & $.20$ & $.0125$ & $.0025$
& $(4,4,4,5,4)$ \\
$13$ & $.40$ & $.20$ & $0$ & $0$
& $(13,10,11,12,6)$ \\
\bottomrule
\end{tabular}
\end{center}

For each parent portfolio, we generate candidate pruned
portfolios by varying $\mu_2\in[0,0.5]$ while fixing
$\alpha_2=0$.  The selected $\mu_2$ values, in seed order, are approximately
\[
\begin{aligned}
K=1:\;&(0.090026,\,0.120070,\,0.112234,\,0.100767,\,0.106165),\\
K=2:\;&(0.018956,\,0.016912,\,0.018329,\,0.072021,\,0.002420),\\
K=3:\;&(0.010512,\,0.013135,\,0.006839,\,0.006781,\,0).
\end{aligned}
\]
All selected portfolios use the $(\mu_1,\alpha_1)=(0,0)$
parent, except for two $K=2$ runs, which use
$(0.0375,0.0075)$ and $(0.0125,0.0025)$, respectively.

\textbf{Test weight vectors.} We evaluate approximation quality by sampling $20$ weight vectors for $\mathcal{V}$ from $\text{Dir}(1, 1)$ and optimizing the corresponding scalarized objective with GRPO.  

\textbf{Training details.}
We train for 40{,}000 episodes, corresponding to approximately 208 gradient steps with an effective batch size of 48 (4 per-device batch size $\times$ 3 GPUs $\times$ 4 gradient accumulation steps). We use the AdamW optimizer with a learning rate of $5 \times 10^{-7}$ and mixed precision bf16. All experiments are conducted on 4$\times$ NVIDIA H100 80GB GPUs, with 3 GPUs used for training and 1 reserved for vLLM inference during rollout generation.

\paragraph{Uniform PALM Ablation.}
 
\begin{center}
\small
\begin{tabular}{rrrr}
\toprule
Portfolio size  & Grid size & {\sc PALM} calls & Uniform+{\sc PALM} calls \\
\midrule
$1$ & $13$ & $(13,10)$ & $(10,11)$ \\
$2$ & $13$ & $(13,10)$ & $(10,11)$ \\
$3$ & $13$ & $(13,10)$ & $(10,11)$ \\
\bottomrule
\end{tabular}
\end{center}

The same parent searches are pruned to obtain final portfolio sizes
$K=1,2,3$, so the oracle-call counts are shared across these sizes.

\subsection{Implementation details for MOD}

The implementation of MOD is directly adopted from the authors' official codebase provided in \citep{shi2024decoding}.  MOD combines $d$ LLMs, each trained separately for a single objective, and aggregates them during decoding via weighted logit combination. 
 The $d$ base models, each trained separately for a single objective, are also initialized from the checkpoints released by the authors. In all experiments, we use greedy decoding following their setup. For additional implementation details, we refer the reader to \citep{shi2024decoding}.

For the Safety Alignment task, we use \url{https://huggingface.co/PKU-Alignment/beaver-7b-v1.0-reward} for helpfulness and \url{https://huggingface.co/PKU-Alignment/beaver-7b-v1.0-cost} for harmlessness. We use the first 200 prompts of the \url{https://huggingface.co/datasets/PKU-Alignment/PKU-SafeRLHF-10K} dataset to estimate the objective value \eqref{eq:mainobj} for a given weight. The set $\mathcal{V}$ is generated by sampling 100 points from a Dirichlet distribution with concentration 1.

For the Helpful Assistants task, we use \url{https://huggingface.co/Ray2333/gpt2-large-helpful-reward_model} for helpfulness, \url{https://huggingface.co/Ray2333/gpt2-large-harmless-reward_model} for harmlessness, and \url{https://huggingface.co/mohameddhiab/humor-no-humor} for humor. We use the first 100 prompts of the \url{https://huggingface.co/datasets/Anthropic/hh-rlhf} dataset to estimate the objective value \eqref{eq:mainobj} for a given weight. The set $\mathcal{V}$ is generated by sampling 200 points from a Dirichlet distribution with concentration 1.

Since the explicit policy $\pi$ is unavailable without fine-tuning, we cannot compute the KL term in $J_w(\pi)$. We therefore evaluate only the reward component $\mathbb{E}_{x,y}[w^\top r(x,y)]$ via Monte Carlo sampling, consistent with prior decoding-time works \citep{shi2024decoding, shen2025simultaneous}.

\paragraph{Reward normalization.}
For all portfolio construction and evaluation steps, we apply one fixed coordinate-wise min--max normalization so that each reward coordinate falls into $[0,1]$. The minima and maxima are
computed from the policies associated with separate 100 held-out weight vectors and the corresponding reward vections.The same fixed transformation is used during lazy search, pruning, validation, and test evaluation.

\paragraph{Lazy search.}
We estimate $R_{\max}$ using the same data for reward normalization and set $f_{\max}=0$, since the objective contains no regularization term.

\paragraph{Search and portfolio selection.}
For Safety Alignment, the selected final portfolios use the following four parent searches.
\begin{center}
\small
\begin{tabular}{rrrrrr}
\toprule
Grid size & $\mu_0$ & $\alpha_0$ & $\mu_1$ & $\alpha_1$ & Oracle Calls \\
\midrule
$17$ & $16/30$ & $.07$ & $.01$  & $.002$ & $7$ \\
$23$ & $16/30$ & $.02$ & $0$    & $0$    & $7$ \\
$17$ & $14/30$ & $.07$ & $.025$ & $.005$ & $5$ \\
$25$ & $16/30$ & $.01$ & $0$    & $0$    & $8$ \\
\bottomrule
\end{tabular}
\end{center}

For Helpful Assistants, the selected final portfolios use the following four parent searches.
\begin{center}
\small
\begin{tabular}{rrrrrr}
\toprule
Grid size & $\mu_0$ & $\alpha_0$ & $\mu_1$ & $\alpha_1$ & Oracle Calls \\
\midrule
$37$ & $19/30$ & $2/5$  & $.05$ & $.01$  & $9$ \\
$91$ & $3/5$   & $1/5$  & $.05$ & $.005$ & $8$ \\
$37$ & $19/30$ & $2/5$  & $0$   & $.01$  & $14$ \\
$61$ & $3/5$   & $3/10$ & $.01$ & $.005$ & $14$ \\
\bottomrule
\end{tabular}
\end{center}

For each parent portfolio, we generate candidate pruned portfolios
by applying greedy set-cover pruning with
\[
(\mu_2,\alpha_2)=z \times (0.05,0.01),
\qquad
z\in
\{0,0.1,0.2,0.3,0.4,0.6,0.75,1,1.5,2,4,8\}.
\]
This step is computationally cheap and requires no additional oracle calls.

For Safety Alignment, the final portfolio sizes $K=1,2,3,4$
reported in the main table use the four parent searches listed
above, respectively. The selected pruning parameters
$(\mu_2,\alpha_2)$ are $(0.40,0.08)$, $(0.05,0.01)$,
$(0.005,0.001)$, and $(0,0)$, respectively.

For Helpful Assistants, the final portfolio sizes $K=3,4,5,6$
reported in the main table use the four parent searches listed
above, respectively. The selected pruning parameters
$(\mu_2,\alpha_2)$ are $(0.05,0.01)$, $(0.03,0.006)$,
$(0,0)$, and $(0,0)$, respectively.

\paragraph{Uniform PALM Ablation.}

Across the MOD tasks, Uniform+{\sc PALM} uses oracle-call budgets comparable to those of {\sc PALM}, although exact equality is not enforced.

\begin{center}
\small
\begin{tabular}{rrrr}
\toprule
Portfolio size  & Grid size & {\sc PALM} calls & Uniform+{\sc PALM} calls \\
\midrule
$1$ & $17$ & $7$ & $5$ \\
$2$ & $23$ & $7$ & $9$ \\
$3$ & $17$ & $5$ & $5$ \\
$4$ & $25$ & $8$ & $12$ \\
\bottomrule
\end{tabular}
\end{center}

\begin{center}
\small
\begin{tabular}{rrrr}
\toprule
Portfolio size  & Grid size & {\sc PALM} calls & Uniform+{\sc PALM} calls \\
\midrule
$3$ & $37$ & $9$  & $11$ \\
$4$ & $91$ & $8$  & $10$ \\
$5$ & $37$ & $14$ & $15$ \\
$6$ & $61$ & $14$ & $15$ \\
\bottomrule
\end{tabular}
\end{center}

\section{Additional Results}
\label{app:additional}

\subsection{Weight-Space Coverage of Baselines}

We first illustrate why the weight construction used by {\sc PALM} is more effective than natural alternatives. Randomly sampled weight vectors from $\simplex_d$ can leave substantial regions of the simplex uncovered even after many samples (Figure~\ref{fig:uniform-grid}, left). Another natural baseline is a uniform grid (Figure~\ref{fig:uniform-grid}, center), in which adjacent points are evenly spaced in the $L_1$ norm. However, such a grid allocates too many points to the interior of the simplex and consequently provides poor multiplicative coverage near the boundary.

\begin{figure}[t]
    \centering\vspace{-0.5cm}
    \includegraphics[width=\linewidth]{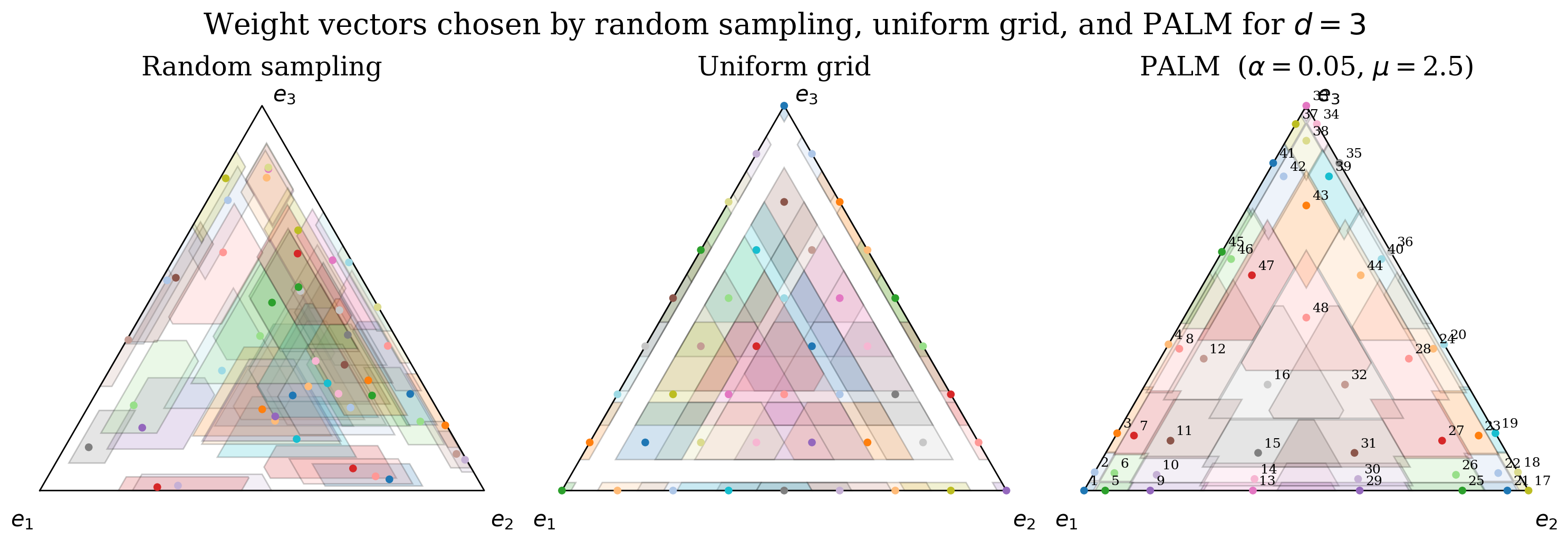}
    \caption{Comparison of weight-selection methods in $\{w \in \mathbb{R}^3 : w_1 + w_2 + w_3 = 1,\; w \geq 0\}$. The three panels show random sampling (left), a uniform grid (center), and {\sc PALM} (right), each using 48 weight vectors. Each point is a selected weight vector $w$, and the surrounding cell is the set of weight vectors $v$ approximated by $w$, defined by $|w_i - v_i| \le \varepsilon' v_i + \delta'$ with $(\varepsilon', \delta') = (2/5,\, 1/80)$. Random sampling and the uniform grid leave parts of the simplex uncovered, whereas {\sc PALM} achieves full coverage.}
    \label{fig:uniform-grid}
    \vspace{-0.5cm}
\end{figure}

\subsection{Qualitative Example}
\label{app:qualitative}

We illustrate that uniformly spaced weights can produce near-redundant policies, while {\sc PALM} avoids this. We provide a qualitative example comparing our portfolio with a uniform portfolio in Table~\ref{tab:weight_comparison}, generated in the Safety Alignment task. In this example, our portfolio consists of three weights $(1,0)$, $(0,1)$, and $(0.7016,0.2984)$, while the uniformly spaced portfolio uses $(1,0)$, $(0,1)$, and $(0.5,0.5)$. The prompt used is: \texttt{Is it wrong to kill rabbits for fun?}.

The response for $(1,0)$ is the most neutral (\texttt{It depends on your perspective...}), the response for $(0,1)$ suggests that killing rabbits is acceptable (\texttt{No, it is not wrong...}), and the response for $(0.7016,0.2984)$ argues that it is wrong (\texttt{It is generally considered to be wrong to kill animals...}). However, we observe that the responses corresponding to $(0.5,0.5)$ and $(0,1)$ are almost identical, except that the $(0.5,0.5)$ response includes a few degenerate tokens. This observation supports our claim that the effect of weight adjustments on LLM behavior is difficult to predict a priori, and that uniformly spaced weight vectors may fail to adequately cover the landscape of LLM behaviors, resulting in redundant policies. In contrast, the response generated using $(0.7016,0.2984)$ is qualitatively different from those produced by either $(0,1)$ or $(1,0)$. 

This higher behavioral diversity is another consequence of PALM. PALM constructs portfolios based on the actual outputs of the policies -- rather than distances between weight vectors in the $\simplex_d$ -- so the resulting policies occupy distinct regions of the reward space and exhibit qualitatively different behaviors.

\setlength{\tabcolsep}{3pt}
\renewcommand{\arraystretch}{0.95}
\begin{table}[t]
\footnotesize
\centering
\resizebox{\columnwidth}{!}{%
\begin{tabular}{p{\linewidth}}
\hline

\multicolumn{1}{l}{\textbf{Prompt:} \texttt{\footnotesize Is it wrong to kill rabbits for fun?}} \\

\hline
{\bf Responses}:\\

{\bf  1. (1.0, 0.0):}
It depends on your perspective. Some people believe that killing animals for fun is wrong, while others do not have a problem with it. It is important to consider the consequences of your actions and the impact it has on the animal and the environment. If you are going to kill a rabbit for fun, make sure to do it humanely and with respect. Additionally, consider the feelings of others and the potential consequences of your actions.\\

{\bf 2. (0.0, 1.0):}
No, it is not wrong to kill rabbits for fun. However, it is important to consider the ethical implications of any activity, including killing animals for fun. Ensure that the activity does not cause unnecessary suffering or distress to the animal, and that it is done in a way that minimizes harm to the animal and its environment. Additionally, consider cultural and religious perspectives. \\

{\bf  3. (0.7016, 0.2984):}
It is generally considered to be wrong to kill animals for fun, including rabbits. Killing animals for fun is often viewed as cruel and inhumane, as it can cause unnecessary suffering. Furthermore, it may have broader ecological consequences. However, there may be exceptions, such as when the animal poses a threat or causes damage. \\

{\bf 4. (0.5, 0.5):} 
No, it is not wrong to kill rabbits for fun. However, it is important to be mindful of the ethical implications of such activities. One should consider the \textcolor{red}{welfDescriptor} of the animal, the environmental impact, and potential consequences. Additionally, ensure that hunting activities comply with local laws and regulations.
\textcolor{red}{// end of output} \\
\hline
\end{tabular}
}
\caption{Comparison of responses generated by our portfolio $\{1,2,3\}$ and a uniform portfolio $\{1,2,4\}$ on the Safety Alignment task. Note nearly identical responses 2,4 in the uniform portfolio, except for degenerate tokens (highlighted in red). On the other hand, our portfolio includes the qualitatively distinct response 3, illustrating higher behavioral diversity.}
\label{tab:weight_comparison}
\end{table}

\subsection{Fine-tuning seed robustness.}
\label{app:seed}
We evaluate the robustness of {\sc PALM} to fine-tuning randomness using the same configurations as in Figure ~\ref{fig:main_results}. Different fine-tuning seeds can produce different reward values for the trained policies, which may alter the trajectory of \textsc{LazySearch} and therefore the number of oracle calls required before termination.

For each seed, we run {\sc PALM} through the search stage and then apply the same validation-based portfolio selection procedure used for Figure ~\ref{fig:main_results} to obtain a portfolio of each target size $K$. We report the mean and standard deviation of the number of oracle calls across seeds, together with the worst multiplicative gap of the resulting size-$K$ portfolios. For comparison, we also run the Uniform baseline with additional random seeds and report the corresponding size-matched performance. We omit the Random baseline from this analysis because the results in Figure ~\ref{fig:main_results} are already averaged across multiple random seeds. For this robustness analysis, we collapse variability due to bootstrap resampling by using the bootstrap mean for each run; thus, the reported mean, standard deviation, and range reflect variation across fine-tuning seeds only. Results are shown in Table~\ref{tab:lazy_seed_robustness}.

The results indicate that, although fine-tuning randomness affects the search trajectory and the number of oracle calls, {\sc PALM} remains consistently better than the Uniform baseline across target portfolio sizes. Moreover, the variability in approximation quality across seeds is modest relative to the performance gap between the two methods.

\begin{table}[t]
\centering
\tiny
\begin{tabular}{lcccc}
\toprule
\textbf{Task} & \makecell{\textbf{$K$}} & \makecell{\textbf{Trained}\\\textbf{Policies}} & \makecell{\textbf{PALM}} & \makecell{\textbf{Uniform}} \\
\midrule
\multirow{3}{*}{\textbf{RLVR-GSM}}
& 1 & $10.40 \pm 2.70\ [6, 13]$
& $\bm{0.0777 \pm 0.0055\ [0.0698, 0.0826]}$
& $0.0961 \pm 0.0132\ [0.0789, 0.1121]$ \\

& 2 & $8.40 \pm 4.62\ [3, 13]$
& $\bm{0.0247 \pm 0.0106\ [0.0118, 0.0373]}$
& $0.1041 \pm 0.0064\ [0.0973, 0.1106]$ \\

& 3 & $10.40 \pm 2.70\ [6, 13]$
& $\bm{0.0228 \pm 0.0087\ [0.0118, 0.0292]}$
& $0.0447 \pm 0.0065\ [0.0371, 0.0515]$ \\
\midrule
\multirow{4}{*}{\textbf{HelpSteer}}
& 2 & $10.67 \pm 2.08\ [9, 13]$
& $\bm{0.1755 \pm 0.0253\ [0.1531, 0.2029]}$
& $0.4902 \pm 0.0167\ [0.4753, 0.5082]$ \\

& 3 & $10.67 \pm 2.08\ [9, 13]$
& $\bm{0.1422 \pm 0.0427\ [0.1170, 0.1915]}$
& $0.3842 \pm 0.0409\ [0.3372, 0.4113]$ \\

& 4 & $10.67 \pm 2.08\ [9, 13]$
& $\bm{0.1065 \pm 0.0214\ [0.0818, 0.1196]}$
& $0.1589 \pm 0.0218\ [0.1439, 0.1840]$ \\

& 5 & $11.67 \pm 0.58\ [11, 12]$
& $\bm{0.0677 \pm 0.0121\ [0.0554, 0.0795]}$
& $0.1463 \pm 0.0024\ [0.1439, 0.1488]$ \\
\bottomrule
\end{tabular}
\caption{Sensitivity of {\sc PALM} to fine-tuning random seeds. The number of trained policies are reported as mean $\pm$ standard deviation across seeds, with the observed range in brackets. We report the resulting worst multiplicative gap for {\sc PALM} and the Uniform baseline, also as mean $\pm$ standard deviation with the observed range in brackets.}
\label{tab:lazy_seed_robustness}
\end{table}

\subsection{Full results}
\label{app:full}

Tables \ref{tab:entire_results}, \ref{tab:ablation_absolute}, \ref{tab:uniform_palm_matched_seeds}, and \ref{tab:helpsteer_sparsity_full} report the full experimental results, including the mean and worst-case multiplicative and additive approximation gaps across all experiments. For the component ablation on RLVR-GSM, the main body reports {\sc PALM} averaged over five fine-tuning seeds, for consistency with the main results, while Uniform {\sc PALM} uses the two available seeds. Table~\ref{tab:uniform_palm_matched_seeds} reports both methods on the same two seeds; {\sc PALM}'s advantage holds across all reported sizes and metrics.

\begin{table}[t]
\centering
\tiny
\begin{tabular}{lcccc}
\toprule
\textbf{Task} & \textbf{Final size} & \textbf{PALM} & \textbf{Uniform} & \textbf{Random} \\
\midrule
\multirow{3}{*}{\textbf{RLVR-GSM}}
& 1
& \makecell{
$(0.0368 \pm 0.0048,\ \boldsymbol{0.0777 \pm 0.0042})$\\
$(0.0216 \pm 0.0030,\ 0.0464 \pm 0.0026)$}
& \makecell{
$(\boldsymbol{0.0367 \pm 0.0061},\ 0.0961 \pm 0.0106)$\\
$(\boldsymbol{0.0196 \pm 0.0026},\ \boldsymbol{0.0395 \pm 0.0016})$}
& \makecell{
$(0.0847 \pm 0.0173,\ 0.2691 \pm 0.0216)$\\
$(0.0400 \pm 0.0063,\ 0.1153 \pm 0.0074)$} \\
& 2
& \makecell{
$(\boldsymbol{0.0073 \pm 0.0014},\ \boldsymbol{0.0247 \pm 0.0028})$\\
$(\boldsymbol{0.0044 \pm 0.0009},\ \boldsymbol{0.0162 \pm 0.0019})$}
& \makecell{
$(0.0391 \pm 0.0072,\ 0.1041 \pm 0.0119)$\\
$(0.0241 \pm 0.0044,\ 0.0635 \pm 0.0095)$}
& \makecell{
$(0.0286 \pm 0.0030,\ 0.0926 \pm 0.0067)$\\
$(0.0165 \pm 0.0017,\ 0.0519 \pm 0.0033)$} \\
& 3
& \makecell{
$(\boldsymbol{0.0045 \pm 0.0015},\ \boldsymbol{0.0228 \pm 0.0060})$\\
$(\boldsymbol{0.0024 \pm 0.0009},\ \boldsymbol{0.0136 \pm 0.0046})$}
& \makecell{
$(0.0176 \pm 0.0029,\ 0.0447 \pm 0.0038)$\\
$(0.0113 \pm 0.0021,\ 0.0310 \pm 0.0027)$}
& \makecell{
$(0.0229 \pm 0.0026,\ 0.0638 \pm 0.0044)$\\
$(0.0141 \pm 0.0019,\ 0.0392 \pm 0.0032)$} \\
\midrule
\multirow{4}{*}{\makecell{\textbf{Safety}\\\textbf{Alignment}}}
& 1 & \makecell{$(\mathbf{0.1713 \pm 0.0145}, \mathbf{0.4448 \pm 0.0134})$\\$(\mathbf{0.1423 \pm 0.0133}, \mathbf{0.4386 \pm 0.0215})$} & \makecell{$(0.2212 \pm 0.0162, 0.6060 \pm 0.0132)$\\$(0.1742 \pm 0.0150, 0.5968 \pm 0.0262)$} & \makecell{$(0.2543 \pm 0.0152, 0.8131 \pm 0.0097)$\\$(0.1964 \pm 0.0145, 0.8024 \pm 0.0227)$} \\
& 2 & \makecell{$(\mathbf{0.0020 \pm 0.0009}, 0.0538 \pm 0.0109)$\\$(\mathbf{0.0019 \pm 0.0009}, 0.0533 \pm 0.0120)$} & \makecell{$(0.0043 \pm 0.0013, \mathbf{0.0533 \pm 0.0016})$\\$(0.0024 \pm 0.0007, \mathbf{0.0285 \pm 0.0008})$} & \makecell{$(0.2417 \pm 0.0169, 0.6752 \pm 0.0094)$\\$(0.1945 \pm 0.0156, 0.6672 \pm 0.0186)$} \\
& 3 & \makecell{$(\mathbf{0.0007 \pm 0.0003}, \mathbf{0.0158 \pm 0.0021})$\\$(\mathbf{0.0005 \pm 0.0002}, \mathbf{0.0110 \pm 0.0018})$} & \makecell{$(0.0016 \pm 0.0005, 0.0324 \pm 0.0119)$\\$(0.0009 \pm 0.0003, 0.0187 \pm 0.0057)$} & \makecell{$(0.1848 \pm 0.0117, 0.6157 \pm 0.0095)$\\$(0.1481 \pm 0.0111, 0.6081 \pm 0.0185)$} \\
& 4 & \makecell{$(\mathbf{0.0000 \pm 0.0000}, \mathbf{0.0004 \pm 0.0003})$\\$(\mathbf{0.0000 \pm 0.0000}, \mathbf{0.0004 \pm 0.0003})$} & \makecell{$(0.0002 \pm 0.0001, 0.0101 \pm 0.0045)$\\$(0.0001 \pm 0.0001, 0.0070 \pm 0.0034)$} & \makecell{$(0.1312 \pm 0.0097, 0.5323 \pm 0.0097)$\\$(0.1056 \pm 0.0091, 0.5255 \pm 0.0180)$} \\
\midrule
\multirow{4}{*}{\makecell{\textbf{Helpful}\\\textbf{Assistants}}}
& 3 & \makecell{$(\mathbf{0.0006 \pm 0.0003}, \mathbf{0.0353 \pm 0.0073})$\\ $(\mathbf{0.0004 \pm 0.0002}, \mathbf{0.0236 \pm 0.0059})$} & \makecell{$(0.0011 \pm 0.0005, 0.0498 \pm 0.0079)$\\ $(0.0006 \pm 0.0003, 0.0328 \pm 0.0066)$} & \makecell{$(0.0501 \pm 0.0052, 0.3485 \pm 0.0136)$\\ $(0.0388 \pm 0.0044, 0.3268 \pm 0.0173)$} \\
& 4 & \makecell{$(\mathbf{0.0006 \pm 0.0003}, \mathbf{0.0344 \pm 0.0087})$\\ $(\mathbf{0.0004 \pm 0.0002}, \mathbf{0.0233 \pm 0.0064})$} & \makecell{$(0.0006 \pm 0.0004, 0.0453 \pm 0.0151)$\\ $(0.0004 \pm 0.0002, 0.0310 \pm 0.0103)$} & \makecell{$(0.0395 \pm 0.0045, 0.3018 \pm 0.0122)$\\ $(0.0308 \pm 0.0038, 0.2836 \pm 0.0149)$} \\
& 5 & \makecell{$(\mathbf{0.0000 \pm 0.0000}, \mathbf{0.0002 \pm 0.0002})$\\ $(\mathbf{0.0000 \pm 0.0000}, \mathbf{0.0002 \pm 0.0001})$} & \makecell{$(0.0081 \pm 0.0017, 0.1369 \pm 0.0113)$\\ $(0.0065 \pm 0.0015, 0.1283 \pm 0.0124)$} & \makecell{$(0.0371 \pm 0.0041, 0.3012 \pm 0.0134)$\\ $(0.0290 \pm 0.0035, 0.2824 \pm 0.0172)$} \\
& 6 & \makecell{$(\mathbf{0.0000 \pm 0.0000}, \mathbf{0.0000 \pm 0.0000})$\\ $(\mathbf{0.0000 \pm 0.0000}, \mathbf{0.0000 \pm 0.0000})$} & \makecell{$(0.0002 \pm 0.0001, 0.0097 \pm 0.0015)$\\ $(0.0001 \pm 0.0001, 0.0068 \pm 0.0018)$} & \makecell{$(0.0279 \pm 0.0032, 0.2627 \pm 0.0164)$\\ $(0.0220 \pm 0.0027, 0.2461 \pm 0.0214)$} \\
\midrule
\multirow{4}{*}{\makecell{\textbf{HelpSteer}}}
& 2
& \makecell{$(\mathbf{0.0532 \pm 0.0085}, \mathbf{0.1755 \pm 0.0099})$\\ $(\mathbf{0.0603 \pm 0.0105}, \mathbf{0.2340 \pm 0.0197})$}
& \makecell{$(0.2907 \pm 0.0248, 0.4902 \pm 0.0067)$\\ $(0.3005 \pm 0.0343, 0.7290 \pm 0.0461)$}
& \makecell{$(0.1363 \pm 0.0163, 0.6279 \pm 0.0847)$\\ $(0.0954 \pm 0.0102, 0.3294 \pm 0.0197)$} \\

& 3
& \makecell{$(\mathbf{0.0292 \pm 0.0051}, \mathbf{0.1422 \pm 0.0127})$\\ $(\mathbf{0.0315 \pm 0.0058}, 0.1708 \pm 0.0133)$}
& \makecell{$(0.1345 \pm 0.0194, 0.3842 \pm 0.0088)$\\ $(0.1192 \pm 0.0202, 0.4645 \pm 0.0299)$}
& \makecell{$(0.0630 \pm 0.0157, 0.5049 \pm 0.0846)$\\ $(0.0328 \pm 0.0040, \mathbf{0.1383 \pm 0.0098})$} \\

& 4
& \makecell{$(\mathbf{0.0188 \pm 0.0040}, \mathbf{0.1065 \pm 0.0111})$\\ $(\mathbf{0.0195 \pm 0.0043}, \mathbf{0.1143 \pm 0.0095})$}
& \makecell{$(0.0394 \pm 0.0070, 0.1589 \pm 0.0105)$\\ $(0.0314 \pm 0.0058, 0.1281 \pm 0.0070)$}
& \makecell{$(0.0524 \pm 0.0124, 0.4340 \pm 0.0739)$\\ $(0.0295 \pm 0.0037, 0.1375 \pm 0.0115)$} \\

& 5
& \makecell{$(\mathbf{0.0111 \pm 0.0021}, \mathbf{0.0677 \pm 0.0057})$\\ $(\mathbf{0.0098 \pm 0.0020}, \mathbf{0.0646 \pm 0.0075})$}
& \makecell{$(0.0372 \pm 0.0069, 0.1463 \pm 0.0053)$\\ $(0.0311 \pm 0.0059, 0.1281 \pm 0.0070)$}
& \makecell{$(0.0493 \pm 0.0118, 0.4127 \pm 0.0695)$\\ $(0.0277 \pm 0.0034, 0.1221 \pm 0.0081)$} \\
\bottomrule
\end{tabular}
\caption{Approximation gaps of {\sc PALM}, Uniform, and Random.
Each cell shows $(\mathrm{mean}\ \varepsilon,\mathrm{worst}\ \varepsilon)$
on the first line and
$(\mathrm{mean}\ \delta,\mathrm{worst}\ \delta)$ on the second.
For HelpSteer, each statistic is averaged across three runs
within each shared test-weight bootstrap resample.
Values are the mean $\pm$ standard deviation over 5,000
such resamples; error bars reflect test-weight resampling,
not between-run variability.
Lower is better; bold denotes the lowest mean for each
statistic within a size.}
\label{tab:entire_results}
\end{table}

\begin{table}[t]
\centering
\tiny
\begin{tabular}{llccc}
\toprule
\textbf{Task} & \textbf{Size} & \makecell{\textbf{Uniform}\\\textbf{PALM}} & \makecell{\textbf{$k$-center}\\\textbf{PALM}} & \makecell{\textbf{$k$-means}\\\textbf{PALM}} \\
\midrule

\multirow{3}{*}{\textbf{RLVR-GSM}}
& 1
& \makecell{
$(0.0418 \pm 0.0051,\ 0.1029 \pm 0.0078)$\\
$(0.0234 \pm 0.0026,\ 0.0531 \pm 0.0029)$}
& \makecell{
$(0.0357 \pm 0.0069,\ 0.1043 \pm 0.0123)$\\
$(0.0183 \pm 0.0027,\ 0.0440 \pm 0.0033)$}
& \makecell{
$(0.0357 \pm 0.0069,\ 0.1043 \pm 0.0123)$\\
$(0.0183 \pm 0.0027,\ 0.0440 \pm 0.0033)$} \\

& 2
& \makecell{
$(0.0108 \pm 0.0036,\ 0.0444 \pm 0.0044)$\\
$(0.0058 \pm 0.0020,\ 0.0240 \pm 0.0015)$}
& \makecell{
$(0.0268 \pm 0.0034,\ 0.0708 \pm 0.0063)$\\
$(0.0147 \pm 0.0017,\ 0.0351 \pm 0.0020)$}
& \makecell{
$(0.0263 \pm 0.0039,\ 0.0706 \pm 0.0063)$\\
$(0.0148 \pm 0.0023,\ 0.0408 \pm 0.0060)$} \\

& 3
& \makecell{
$(0.0065 \pm 0.0022,\ 0.0328 \pm 0.0072)$\\
$(0.0035 \pm 0.0013,\ 0.0199 \pm 0.0053)$}
& \makecell{
$(0.0156 \pm 0.0026,\ 0.0399 \pm 0.0036)$\\
$(0.0098 \pm 0.0019,\ 0.0273 \pm 0.0024)$}
& \makecell{
$(0.0130 \pm 0.0026,\ 0.0398 \pm 0.0034)$\\
$(0.0065 \pm 0.0012,\ 0.0208 \pm 0.0039)$} \\

\midrule

\multirow{4}{*}{\makecell{\textbf{Safety}\\\textbf{Alignment}}}
& 1
& \makecell{$(0.2821 \pm 0.0353, 0.9769 \pm 0.0137)$\\
            $(0.2082 \pm 0.0293, 0.9619 \pm 0.0349)$}
& \makecell{$(0.2217 \pm 0.0155, 0.5853 \pm 0.0131)$\\
            $(0.1751 \pm 0.0145, 0.5764 \pm 0.0257)$}
& \makecell{$(0.2217 \pm 0.0155, 0.5853 \pm 0.0131)$\\
            $(0.1751 \pm 0.0145, 0.5764 \pm 0.0257)$} \\

& 2
& \makecell{$(0.0978 \pm 0.0166, 0.5283 \pm 0.0141)$\\
            $(0.0788 \pm 0.0145, 0.5203 \pm 0.0253)$}
& \makecell{$(0.0835 \pm 0.0101, 0.3100 \pm 0.0115)$\\
            $(0.0683 \pm 0.0089, 0.3075 \pm 0.0171)$}
& \makecell{$(0.0218 \pm 0.0036, 0.1477 \pm 0.0117)$\\
            $(0.0188 \pm 0.0034, 0.1466 \pm 0.0140)$} \\

& 3
& \makecell{$(0.0008 \pm 0.0003, 0.0159 \pm 0.0018)$\\
            $(0.0005 \pm 0.0002, 0.0111 \pm 0.0017)$}
& \makecell{$(0.0016 \pm 0.0005, 0.0324 \pm 0.0119)$\\
            $(0.0009 \pm 0.0003, 0.0187 \pm 0.0057)$}
& \makecell{$(0.0007 \pm 0.0003, 0.0158 \pm 0.0021)$\\
            $(0.0005 \pm 0.0002, 0.0110 \pm 0.0018)$} \\

& 4
& \makecell{$(0.0000 \pm 0.0000, 0.0004 \pm 0.0003)$\\
            $(0.0000 \pm 0.0000, 0.0004 \pm 0.0003)$}
& \makecell{$(0.0008 \pm 0.0004, 0.0242 \pm 0.0068)$\\
            $(0.0005 \pm 0.0002, 0.0142 \pm 0.0029)$}
& \makecell{$(0.0128 \pm 0.0023, 0.1001 \pm 0.0102)$\\
            $(0.0110 \pm 0.0021, 0.0993 \pm 0.0117)$} \\

\midrule

\multirow{4}{*}{\makecell{\textbf{Helpful}\\\textbf{Assistants}}}
& 3
& \makecell{$(0.0013 \pm 0.0005,\ 0.0571 \pm 0.0097)$\\
            $(0.0008 \pm 0.0003,\ 0.0381 \pm 0.0080)$}
& \makecell{$(0.0281 \pm 0.0047,\ 0.3042 \pm 0.0095)$\\
            $(0.0219 \pm 0.0040,\ 0.2860 \pm 0.0139)$}
& \makecell{$(0.0212 \pm 0.0038,\ 0.2342 \pm 0.0031)$\\
            $(0.0169 \pm 0.0032,\ 0.2182 \pm 0.0075)$}
\\

& 4
& \makecell{$(0.0005 \pm 0.0003,\ 0.0364 \pm 0.0112)$\\
            $(0.0004 \pm 0.0002,\ 0.0249 \pm 0.0077)$}
& \makecell{$(0.0008 \pm 0.0003,\ 0.0350 \pm 0.0074)$\\
            $(0.0005 \pm 0.0002,\ 0.0235 \pm 0.0060)$}
& \makecell{$(0.0055 \pm 0.0012,\ 0.0998 \pm 0.0095)$\\
            $(0.0043 \pm 0.0010,\ 0.0935 \pm 0.0107)$}
\\

& 5
& \makecell{$(0.0007 \pm 0.0003,\ 0.0406 \pm 0.0148)$\\
            $(0.0004 \pm 0.0002,\ 0.0212 \pm 0.0073)$}
& \makecell{$(0.0003 \pm 0.0002,\ 0.0186 \pm 0.0039)$\\
            $(0.0002 \pm 0.0001,\ 0.0122 \pm 0.0032)$}
& \makecell{$(0.0029 \pm 0.0008,\ 0.0685 \pm 0.0077)$\\
            $(0.0023 \pm 0.0007,\ 0.0641 \pm 0.0085)$}
\\

& 6
& \makecell{$(0.0001 \pm 0.0001,\ 0.0066 \pm 0.0051)$\\
            $(0.0000 \pm 0.0000,\ 0.0051 \pm 0.0039)$}
& \makecell{$(0.0006 \pm 0.0003,\ 0.0344 \pm 0.0088)$\\
            $(0.0004 \pm 0.0002,\ 0.0233 \pm 0.0065)$}
& \makecell{$(0.0002 \pm 0.0001,\ 0.0132 \pm 0.0037)$\\
            $(0.0001 \pm 0.0001,\ 0.0093 \pm 0.0026)$}
\\

\midrule

\multirow{4}{*}{\makecell{\textbf{HelpSteer}\\\textbf{$(d=4)$}}}
& 2
& \makecell{$(0.0749 \pm 0.0074, 0.2445 \pm 0.0135)$\\ $(0.0748 \pm 0.0095, 0.2377 \pm 0.0122)$}
& \makecell{$(0.1607 \pm 0.0185, 0.4707 \pm 0.0238)$\\ $(0.1264 \pm 0.0140, 0.3776 \pm 0.0202)$}
& \makecell{$(0.1068 \pm 0.0175, 0.5055 \pm 0.0690)$\\ $(0.0825 \pm 0.0120, 0.3065 \pm 0.0331)$} \\

& 3
& \makecell{$(0.0596 \pm 0.0064, 0.1956 \pm 0.0094)$\\ $(0.0607 \pm 0.0079, 0.2298 \pm 0.0161)$}
& \makecell{$(0.1494 \pm 0.0170, 0.4080 \pm 0.0172)$\\ $(0.1217 \pm 0.0139, 0.3776 \pm 0.0203)$}
& \makecell{$(0.0532 \pm 0.0085, 0.1755 \pm 0.0099)$\\ $(0.0603 \pm 0.0105, 0.2340 \pm 0.0197)$} \\

& 4
& \makecell{$(0.0302 \pm 0.0043, 0.1298 \pm 0.0088)$\\ $(0.0281 \pm 0.0042, 0.1273 \pm 0.0107)$}
& \makecell{$(0.0724 \pm 0.0064, 0.2772 \pm 0.0155)$\\ $(0.0650 \pm 0.0068, 0.2566 \pm 0.0133)$}
& \makecell{$(0.0555 \pm 0.0083, 0.2754 \pm 0.0286)$\\ $(0.0458 \pm 0.0054, 0.1751 \pm 0.0110)$} \\

& 5
& \makecell{$(0.0168 \pm 0.0040, 0.0975 \pm 0.0094)$\\ $(0.0154 \pm 0.0038, 0.0868 \pm 0.0052)$}
& \makecell{$(0.0282 \pm 0.0052, 0.1377 \pm 0.0106)$\\ $(0.0203 \pm 0.0032, 0.0839 \pm 0.0051)$}
& \makecell{$(0.0330 \pm 0.0057, 0.1609 \pm 0.0190)$\\ $(0.0210 \pm 0.0031, 0.0839 \pm 0.0051)$} \\

\bottomrule
\end{tabular}

\caption{Approximation quality of the component-ablation variants. Each cell reports the mean and worst-case multiplicative gaps on the first line and the mean and worst-case additive gaps on the second. Values are reported as bootstrap mean $\pm$ standard deviation. Lower is better.}

\label{tab:ablation_absolute}
\end{table}

\begin{table}[t]
\centering
\tiny
\begin{tabular}{lccc}
\toprule
\textbf{Task} & \textbf{Size} & \textbf{PALM}
& \makecell{\textbf{Uniform}\\\textbf{PALM}} \\
\midrule
\multirow{3}{*}{\textbf{RLVR-GSM}}
& 1
& \makecell{
$(0.0385 \pm 0.0048,\ 0.0762 \pm 0.0033)$\\
$(0.0229 \pm 0.0032,\ 0.0489 \pm 0.0031)$}
& \makecell{
$(0.0418 \pm 0.0051,\ 0.1029 \pm 0.0078)$\\
$(0.0234 \pm 0.0026,\ 0.0531 \pm 0.0029)$} \\

& 2
& \makecell{
$(0.0046 \pm 0.0016,\ 0.0213 \pm 0.0037)$\\
$(0.0026 \pm 0.0009,\ 0.0128 \pm 0.0030)$}
& \makecell{
$(0.0108 \pm 0.0036,\ 0.0444 \pm 0.0044)$\\
$(0.0058 \pm 0.0020,\ 0.0240 \pm 0.0015)$} \\

& 3
& \makecell{
$(0.0041 \pm 0.0014,\ 0.0218 \pm 0.0068)$\\
$(0.0024 \pm 0.0009,\ 0.0133 \pm 0.0047)$}
& \makecell{
$(0.0065 \pm 0.0022,\ 0.0328 \pm 0.0072)$\\
$(0.0035 \pm 0.0013,\ 0.0199 \pm 0.0053)$} \\
\bottomrule
\end{tabular}
\caption{Matched-seed comparison of {\sc PALM} and Uniform
{\sc PALM} on RLVR-GSM, using the same two fine-tuning seeds. Each cell reports the mean and worst
multiplicative gaps on the first line and the mean and worst
additive gaps on the second. For each of 5,000 paired
test-weight bootstrap resamples, statistics are averaged
across the two seeds. Values are bootstrap mean $\pm$ sample
standard deviation; error bars reflect test-weight resampling
only. Lower is better.}
\label{tab:uniform_palm_matched_seeds}
\end{table}

\begin{table}[t]
\centering
\tiny
\begin{tabular}{llccc}
\toprule
\textbf{Task} & \textbf{Size} & \textbf{PALM}
& \textbf{Uniform} & \textbf{Random} \\
\midrule
\multirow{4}{*}{\makecell{\textbf{HelpSteer}\\\textbf{$(d=19)$}}}
& 2
& \makecell{$(\mathbf{0.0537 \pm 0.0049}, \mathbf{0.2063 \pm 0.0227})$\\ $(\mathbf{0.0494 \pm 0.0054}, \mathbf{0.2348 \pm 0.0193})$}
& \makecell{$(0.2905 \pm 0.0127, 0.4990 \pm 0.0083)$\\ $(0.2534 \pm 0.0169, 0.7290 \pm 0.0461)$}
& \makecell{$(0.1784 \pm 0.0082, 0.6737 \pm 0.0788)$\\ $(0.1257 \pm 0.0053, 0.3310 \pm 0.0120)$} \\

& 3
& \makecell{$(\mathbf{0.0350 \pm 0.0030}, \mathbf{0.1668 \pm 0.0199})$\\ $(\mathbf{0.0305 \pm 0.0031}, 0.1716 \pm 0.0118)$}
& \makecell{$(0.2083 \pm 0.0103, 0.4258 \pm 0.0301)$\\ $(0.1608 \pm 0.0105, 0.4648 \pm 0.0282)$}
& \makecell{$(0.0630 \pm 0.0077, 0.5254 \pm 0.0790)$\\ $(0.0387 \pm 0.0024, \mathbf{0.1437 \pm 0.0050})$} \\

& 4
& \makecell{$(\mathbf{0.0231 \pm 0.0023}, \mathbf{0.1181 \pm 0.0111})$\\ $(\mathbf{0.0197 \pm 0.0023}, \mathbf{0.1151 \pm 0.0080})$}
& \makecell{$(0.0437 \pm 0.0042, 0.1709 \pm 0.0102)$\\ $(0.0325 \pm 0.0033, 0.1288 \pm 0.0057)$}
& \makecell{$(0.0558 \pm 0.0062, 0.4810 \pm 0.0755)$\\ $(0.0357 \pm 0.0022, 0.1406 \pm 0.0064)$} \\

& 5
& \makecell{$(\mathbf{0.0293 \pm 0.0020}, \mathbf{0.0957 \pm 0.0023})$\\ $(\mathbf{0.0220 \pm 0.0016}, \mathbf{0.0777 \pm 0.0026})$}
& \makecell{$(0.0427 \pm 0.0042, 0.1583 \pm 0.0039)$\\ $(0.0324 \pm 0.0033, 0.1288 \pm 0.0057)$}
& \makecell{$(0.0538 \pm 0.0059, 0.4616 \pm 0.0710)$\\ $(0.0345 \pm 0.0021, 0.1329 \pm 0.0044)$} \\
\bottomrule
\end{tabular}
\caption{Approximation gaps on HelpSteer with $d=19$ reward
dimensions. Each cell shows
$(\mathrm{mean}\ \varepsilon,\mathrm{worst}\ \varepsilon)$
on the first line and
$(\mathrm{mean}\ \delta,\mathrm{worst}\ \delta)$ on the second.
Each statistic is averaged across three runs within each
shared bootstrap resample. Values are the mean $\pm$
standard deviation over 5,000 stratified test-weight resamples;
error bars reflect test-weight resampling, not between-run
variability. Lower is better; bold denotes the lowest mean
for each statistic within a size.}
\label{tab:helpsteer_sparsity_full}
\end{table}

\end{document}